%% file: main.tex
\documentclass{article}
\pdfoutput=1
\usepackage{Custom}
\usepackage{icml2025/fancyhdr}
\usepackage{icml2025/algorithm}
\usepackage{icml2025/algorithmic}
\usepackage[accepted]{icml2025/icml2025}
\usepackage{listings}
\usepackage{multicol}
\usepackage{authblk}
\usepackage{multirow}
\title{\bfseries\Large
    Provably Overwhelming Transformer Models with Designed Inputs
}

\newcommand{\mnote}[1]{{\highlightname{Matt C.}{#1}{blue}}}

\newcommand{\lev}[1]{{\highlightname{Lev}{#1}{purple}}}

\icmltitlerunning{
    Provably Overwhelming Transformer Models with Designed Inputs
}
\begin{document}
\sloppy

\twocolumn[
\icmltitle{
    Provably Overwhelming Transformer Models with Designed Inputs
}
\begin{icmlauthorlist}
	\icmlauthor{Lev Stambler}{Quics,UMD,NT}
	\icmlauthor{Seyed Sajjad Nezhadi}{Quics,UMD,Iluvatar}
	\icmlauthor{Matthew Coudron}{Quics,UMD,NIST}
\end{icmlauthorlist}
\icmlaffiliation{Quics}{Joint Center for Quantum Information and Computer Science, University of Maryland}
\icmlaffiliation{UMD}{Department of Computer Science, University of Maryland}
\icmlaffiliation{NIST}{National Institute of Standards and Technology}
\icmlaffiliation{NT}{Neon Tetra LLC}
\icmlaffiliation{Iluvatar}{iluvatar Technologies}
\icmlcorrespondingauthor{Lev Stambler}{levstamb@umd.edu}

\icmlkeywords{ML Theory, Formal Guarantees, Transformers, Interpretability, Machine Learning, ICML}

\vskip 0.3in
]
\numberwithin{theorem}{section}  

\theoremstyle{plain}     


\input{sections/commands.tex}

\printAffiliationsAndNotice{} 

\begin{abstract}
We develop an algorithm which, given a trained transformer model $\mathcal{M}$ as input, as well as a string of tokens $s$ of length $n_{fix}$ and an integer $n_{free}$, can generate a mathematical proof that $\mathcal{M}$ is ``overwhelmed'' by $s$, in time and space $\widetilde{O}(n_{fix}^2 + n_{free}^3)$.
We say that $\mathcal{M}$ is ``overwhelmed'' by $s$ when the output of the model evaluated on this string plus any additional string $t$, $\mathcal{M}(s + t)$, is completely insensitive to the value of the string $t$ whenever length($t$) $\leq n_{free}$.
Along the way, we prove a particularly strong worst-case form of ``over-squashing'' \cite{alon2021bottleneckgraphneuralnetworks, barbero2024transformers}, which we use to bound the model's behavior.
Our technique uses computer-aided proofs to establish this type of operationally relevant guarantee about transformer models.
We empirically test our algorithm on a single layer transformer complete with an attention head, layer-norm, MLP/ReLU layers, and RoPE positional encoding.
We believe that this work is a stepping stone towards the difficult task of obtaining useful guarantees for trained transformer models.
\end{abstract}


\input{sections/intro.tex}
\input{sections/notation.tex}
\input{sections/model.tex}
\input{sections/framework.tex}

\input{sections/technical_precurse.tex}

\input{sections/model_collapse.tex}

\input{sections/model_eval.tex}
\input{sections/conclusion.tex}

\bibliography{cubebib}

\appendix
\onecolumn
\input{sections/appendix_plots.tex}
\input{sections/appendix_model.tex}

\input{sections/appendix_framework.tex}

\input{sections/appendix_techincal.tex}
\input{sections/appendix_perm.tex}
\input{sections/appendix_det_eps.tex}

\end{document}

%% file: sections/commands.tex
\newcommand{\eps}{\epsilon}

\newcommand{\diag}{\mathrm{diag}}

\newcommand{\nfix}{n_{fix}}
\newcommand{\squash}{\nfix}

\renewcommand\r{\mbox{\bf R}\xspace}
\renewcommand\p{\mbox{\bf P}\xspace}
\newcommand\np{\mbox{\bf NP}\xspace}
\newcommand\po{\mbox{\bf PO}\xspace}
\newcommand\npo{\mbox{\bf NPO}\xspace}
\newcommand\ptas{\mbox{\bf PTAS}\xspace}
\newcommand\apx{\mbox{\bf APX}\xspace}
\newcommand\logapx{\mbox{\bf Log-APX}\xspace}
\newcommand\polyapx{\mbox{\bf Poly-APX}\xspace}
\newcommand\cnp{\mbox{\bf coNP}\xspace}
\newcommand\corp{\mbox{\bf coRP}\xspace}
\newcommand\fp{\mbox{\bf FP}\xspace}
\newcommand\sigmaone{\mbox{\bf $\Sigma_1$}\xspace}
\newcommand\sigmatwo{\mbox{\bf $\Sigma_2$}\xspace}
\newcommand\sigmathree{\mbox{\bf $\Sigma_3$}\xspace}
\newcommand\pione{\mbox{\bf $\Pi_1$}\xspace}
\newcommand\pitwo{\mbox{\bf $\Pi_2$}\xspace}
\newcommand\pithree{\mbox{\bf $\Pi_3$}\xspace}
\newcommand\rp{\mbox{\bf RP}\xspace}
\newcommand\zpp{\mbox{\bf ZPP}\xspace}
\newcommand\bpp{\mbox{\bf BPP}\xspace}
\newcommand\ph{\mbox{\bf PH}\xspace}
\newcommand\pspace{\mbox{\bf PSPACE}\xspace}
\newcommand\npspace{\mbox{\bf NPSPACE}\xspace}
\newcommand\dl{\mbox{\bf L}\xspace}
\newcommand\conl{\mbox{\bf coNL}\xspace}
\newcommand\sharpp{\mbox{\#{\bf P}}\xspace}
\newcommand\parityp{\mbox{$\oplus$ {\bf P}}\xspace}
\renewcommand\ip{\mbox{\bf IP}\xspace}
\newcommand\mip{\mbox{\bf MIP}\xspace}
\newcommand\mipstar{\mbox{\bf MIP*}\xspace}
\newcommand\re{\mbox{\bf RE}\xspace}
\newcommand\pcp{\mbox{\bf PCP}}
\newcommand\dtime{\mbox{\bf DTIME}}
\newcommand\ntime{\mbox{\bf NTIME}}
\newcommand\dspace{\mbox{\bf SPACE}\xspace}
\newcommand\nspace{\mbox{\bf NSPACE}\xspace}
\newcommand\expspace{\mbox{\bf EXPSPACE}\xspace}
\newcommand\cnspace{\mbox{\bf coNSPACE}\xspace}
\newcommand\exptime{\mbox{\bf EXPTIME}\xspace}
\newcommand\nexptime{\mbox{\bf NEXPTIME}\xspace}
\newcommand\genclass{\mbox{$\cal C$}\xspace}
\newcommand\cogenclass{\mbox{\bf co$\cal C$}\xspace}
\newcommand\size{\mbox{\bf SIZE}\xspace}

\newcommand{\calX}{\mathcal{X}}
\newcommand{\calY}{\mathcal{Y}}
\newcommand{\calZ}{\mathcal{Z}}


\newcommand{\dVocab}{{d_{vocab}}}
\newcommand{\dEmb}{{d_{emb}}}
\newcommand{\pQuery}{\vec{p}_{q}}
\newcommand{\tQuery}{\vec{t}_{q}}
\newcommand{\nctx}{{n_{ctx}}}
\newcommand{\nfree}{n_{free}}

\newcommand{\Model}{\mathcal{M}}
\newcommand{\Iden}{\mathbb{I}}
\newcommand{\Embed}{\text{E}}
\newcommand{\Unembed}{\mathrm{Unembed}}
\newcommand{\softmax}{\text{softmax}}
\newcommand{\PosRot}{\Theta}

\newcommand{\relu}{\text{ReLU}}
\newcommand{\MLP}{\text{MLP}}

\newcommand{\LayerNorm}{\texttt{LN}}
\newcommand{\EmbedLN}{\text{E}_{\LayerNorm}}
\newcommand{\OneVec}{\vec{1}}
\newcommand{\OneMat}{\mathbb{1}}
\newcommand{\Expec}{\mathbb{E}}
\newcommand{\Var}{\mathrm{Var}}
\newcommand{\td}[1]{\widetilde{#1}}

\newcommand{\attnH}{\texttt{AttnH}}
\newcommand{\attn}{\texttt{Attn}}
\newcommand{\RoPE}{\operatorname{RoPE}}


\newcommand{\maxOG}{\max}
\newcommand{\minOG}{\min}

\newcommand{\true}{\texttt{True}}
\newcommand{\checkOneLayerDictator}{\texttt{CheckOneLayerDictator}}
\newcommand{\MaxDiff}{\texttt{M}}
\newcommand{\logitDiff}{\ell_{\text{diff}}}
\newcommand{\OneHotSpace}{\mathbb{O}}
\newcommand{\InpSpace}{\OneHotSpace_{des}}
\newcommand{\RepSpace}{\OneHotSpace_{rep}}
\newcommand{\freeSpace}{\mathbb{O}_{free}}
\newcommand{\Lip}{\mathrm{Lip}}
\newcommand{\LipMLP}{\Lip({\MLP})}
\newcommand{\LipModel}{\Lip({\Model})}
\newcommand{\LipAttn}{\Lip({\attnH})}
\newcommand{\LipLayer}{\Lip({\mathrm{layer}})}
\newcommand{\query}{q}

\newcommand{\rep}{r}
\newcommand{\desSet}{{s}}
\newcommand{\desSetSize}{s}
\newcommand{\ndes}{\nfix}
\newcommand{\sndes}{\nfix}

\newcommand{\vece}{\vec{e}}
\newcommand{\vecRep}{\vec{e}_\rep}
\newcommand{\fixedOut}{f}
\newcommand{\vecNctx}{\vec{e}_{\nctx}}
\newcommand{\vecQuery}{\vec{e}_\query}
\newcommand{\vecDes}{\vec{e}_{\rep}}
\newcommand{\vecDesSet}{\vec{e}_{\desSet_1}, \ldots, \vec{e}_{\desSet_{s}}}
\newcommand{\vecDesSetStack}{\vec{e}^T_{\desSet_1}, \ldots, \vec{e}^T_{\desSet_{s}}}
\newcommand{\vecDesB}{\vec{e}_{\rep}'}
\newcommand{\blowupSet}{\mathcal{B}}
\newcommand{\shiftSet}{\mathcal{S}}
\newcommand{\bsSet}{\blowupSet \shiftSet}
\newcommand{\component}{{\texttt{comp}}}
\newcommand{\fenc}{f^{enc}}

\newcommand{\desF}[1]{#1_{(\nfix, \nctx) \mid \desSet, \query}}

\newcommand{\frobInf}{{F_{\infty}}}

\newcommand{\WD}{\mathcal{W}}
\newcommand{\notJ}{{\overline{J}}}

\newcommand{\overw}{overwhelmed}
\newcommand{\Overw}{Overwhelmed}
\newcommand{\OverwQ}{``Overwhelmed''}
\newcommand{\overwQ}{``overwhelmed''}

\newcommand{\freeToks}{\texttt{freeToks}}

%% file: sections/intro.tex
\section{Introduction}

Decoder-only transformers \cite{vaswani2023attentionneed} have become an enormously popular paradigm in the past few years \cite{geminiteam2024geminifamilyhighlycapable, openai2024gpt4technicalreport, Li_2022}.
Yet, our theoretical understanding of these models remains limited.
Proving mathematically rigorous statements about transformers face challenges due to the high dimensionality and complexity of transformers.
To circumvent these problems, current techniques either require simplifying assumption or very specific restrictions on the model/ dataset. 
In this work we propose a different approach focused toward producing rigorous and operationally relevant guarantees for specific, trained, transformer models.
That is, we ask the question:

	\textit{
	Can we develop an algorithm which can provably bound the behavior of a specific trained transformer model?}

With this motivation in mind, we design a class of algorithms which can prove when a particular type of ``\emph{over-squashing}'' phenomenon occurs in a trained transformer model.
Over-squashing is a known phenomenon in graph neural networks (GNNs) in which the representation of distinct nodes becomes arbitrarily close as the number of layers grows \cite{alon2021bottleneckgraphneuralnetworks, barbero2024transformers}.
Recently, Barbero et al.~\cite{barbero2024transformers} have shown that over-squashing occurs in the limit for a transformer model without positional embeddings.

Over-squashing on its own is an important phenomenon.
Yet, it does not directly constrain the model's output behavior.
In this work, we focus on a specific related phenomenon, which we call ``overwhelming''.
A model is ``overwhelmed'' by an input string if the model's output is unchanged by concatenating with \emph{any} new string of tokens of a fixed length.

\begin{definition}[Overwhelm, Overwhelmed, Overwhelming]
    For a given transformer model $\Model$, string of tokens $s$ of length $\nfix$, a fixed final token $q$\footnote{
We find it convenient to fix the final token due to its importance in determining the output logits for the transformer architecture.
We refer to the final token as the ``query token.''}, and integer $\nfree$ we say that $\Model$ is ``overwhelmed" by $s$ if the output of the model evaluated on $s$ plus any additional string $t$ and fixed final token $q$, $\Model(s + t + q)$, is the same regardless of the value of the string $t$ whenever length($t$) $\leq \nfree$.
\end{definition}

In this work, we provide concrete algorithms (\cref{alg:overwhelmCheckDet}, \cref{alg:overwhelmCheck2}) which can produce a proof that a given token sequence overwhelms a given model, as stated in \cref{thm:informal_overwhelm} and \cref{thm:informal_overwhelm_perm} respectively. 
Our algorithms work in two stages: first we upper bound a quantity which we call \emph{worst-case deviation} (denoted by $\WD$), which bounds the extent to which the logit weights output can vary.
Second, we lower-bound a quantity which we call \emph{peak-to-peak difference} which measures the difference in the maximum and second maximum logit weights output by the model.
We thus get our first main result:

\begin{theorem}[Formally restated in \cref{thm:InpRes}]
	\label{thm:informal_overwhelm}
	If Algorithm \ref{alg:overwhelmCheckDet} returns ``Overwhelmed'' when run on a transformer model $\Model$, with a fixed input string of tokens $s$, final token $q$, and context length $n_{ctx}$ then $\Model$ is, provably, overwhelmed by $s$.
\end{theorem}

Proving ``overwhelming'' for transformer models can have many interesting and useful applications.
For example, discovering overwhelming strings can be used to find examples of model hallucination.
If a fixed string causes the output to disregard even a small set of input tokens, the model will produce the wrong result for highly sensitive functions such as parity and finding bugs in code.
Additionally, overwhelming strings can be used to ``jailbreak'' models by forcing the model to ignore part of the system prompt.
Discovering these instances can be relevant in performing safety evaluations of models.
Another application would use overwhelming strings to prove no-go results for prompt engineering: no ``prompt'' of a fixed length can equip a model $\mathcal{M}$ to properly compute a specific function, such as parity. Here the prompt will be the free string and the function input the fixed ``overwhelming'' string. 
Finally, we can also use the frequency of overwhelming strings to define a sort of ``model complexity'' for a given context length; the less overwhelming that occurs for a given context window, the more ``powerful'' a model is for that context window.

Next, we design \cref{alg:overwhelmCheck2}, to achieve performance improvements in a restricted setting in which the $\nfree$ tokens are chosen from the permutation set of a single string (as opposed to arbitrary tokens). This restricted setting is of practical interest because it often occurs with much shorter ``overwhelming'' strings.
Furthermore, the ability of a transformer to distinguish between different permutations of the same string is relevant in many real-world applications.  For example, if an LLM cannot distinguish between the two strings in \cref{fig:perm_example}, then, the model cannot identify the variable $a$ as being defined before it is used.

\begin{figure}[H]
	\centering
	\begin{multicols}{2}
\begin{lstlisting}[language=Python]
# Code snippet
def my_func():
	...
	a = 1
	b = a + 1
	...
	\end{lstlisting}
	\columnbreak
\begin{lstlisting}[language=Python]
# Code snippet
def my_func():
	...
	b = a + 1
	a = 1
	...
\end{lstlisting}
\end{multicols}
\caption{Two snippets of code that are identical except for the order of the assignment statements.}
	\label{fig:perm_example}
\end{figure}

\begin{theorem}[Formally restated in \cref{thm:InpResPermInv}]
	\label{thm:informal_overwhelm_perm}
	If \cref{alg:overwhelmCheck2} returns ``Overwhelmed'' when run on a  transformer model $\Model$, with a fixed input string of tokens $s$, permutation string $x$, query token $q$, and context length $n_{ctx}$ then $\Model$ is, provably, overwhelmed by $s$ when the free string is restricted to be a permutation of $x$.
\end{theorem}
\vspace{-0.5em}

We focus on single-layer transformers in this paper. Our results can be straight forwardly extended to multi-layer models by approximating  Lipschitz constants \cite{kim2021lipschitzconstantselfattention}, though the bounds obtained for higher depth models are too loose to be useful in practice. New ideas are required to obtain tighter bounds for multi-layer transformers.

Before proving these theorems, we first provide a mathematical description of the model in \cref{sec:model}, as well as a metric of over-squashing called worst-case deviation, in \cref{sec:worst_case_framework}.
In \cref{sec:meta_framework}, we provide the main algorithm of this paper and prove \cref{thm:informal_overwhelm}.
In \cref{sec:perm_invar}, we prove \cref{thm:informal_overwhelm_perm} and introduce a linear program which is helpful in bounding extremal values in attention; we believe that this linear program may be of independent interest as well.
In addition, in \cref{sec:convergence}, we prove that, in the asymptotic limit, the model converges to a fixed output model when the fixed string is repeated many times before the free string.

\textbf{Empirical Results:} In \cref{sec:model_eval}, we use our algorithm in practice on a trained, single-layer model.
The algorithm checks for ``overwhelming'' when restricting the free string to an element of one of two permutation classes.
We use fixed strings of varying sizes (up to 1,500 tokens): our fixed strings come from either: (1) the test set of our training data, (2) randomly sampled tokens, (3) a repeated string.
We also test whether ``overwhelming'' persists across multiple steps of token generation.

\subsubsection*{Related Work}
Gross et al.~\cite{gross2024compact} also consider a concrete instance of a trained transformer model and produce a proof that the model can select the maximum token of a sequence with a lower bound on the probability of success.
Barbero et al.~\cite{barbero2024transformers} prove that over-squashing occurs in the limit for a (multi-layer) model without positional embeddings.
Though similar to our work, their results are not directly comparable to ours as they do not provide a concrete algorithm to check for over-squashing in a trained model.
Refs.~\cite{hahn2020theoretical, hahn2024sensitive} use Lipschitz constants to prove various properties of neural networks and provide an implicit algorithm to check for these properties.
In recent year, there has also been a flurry of impossibility results related to transformers \cite{peng2024limitations, merrill2023parallelism, sanford2024representational}, these works focus on general impossibility for the computational class of the transformer architecture.

%% file: sections/model.tex
\section{Model}
\label{sec:model}
In this paper, we focus on a standard decoder-only transformer model which uses RoPE positional encoding, parallel feed-forward residuals, and layer normalization.
Our model closely follows the GPT-NeoX architecture \cite{black2022gptneox20bopensourceautoregressivelanguage} except that we use a single layer-normalization for the MLP and the attention mechanism instead of two
\footnote{Ref.~\cite{black2022gptneox20bopensourceautoregressivelanguage} intended to initially use a single layer-normalization for the MLP and attention mechanism, but, due to a bug, this was changed in the final implementation.
We thus use the intended architecture in this paper.}.

We use relatively standard notation for transformer models where $\LayerNorm$ is the layer normalization function, $\Embed$ is the embedding function, $\Unembed$ is the unembedding function, $\attn$ is the attention mechanism,  $\attnH$ is an attention head,
$\MLP$ is the feed-forward network, and $\Iden$ is the identity function.
Concrete details can be found in \cref{sec:appendix_model}.

We also use the notation in Black et al. \cite{black2022gptneox20bopensourceautoregressivelanguage} where we re-write the effect of RoPE via a rotation matrix $\Theta_{i, j}$ such that for $X \in \OneHotSpace^{\nctx}$, we have
\begin{align*}
	\left(\RoPE(X Q) \cdot \RoPE(K^T X^T)\right)[i, j]
	\\= \vece_i \cdot X Q \cdot \Theta_{i, j} \cdot  K^T  X^T \cdot \vece_j^T.
\end{align*}


The model we study is formally defined as follows:
\newcommand{\ly}[1]{{(#1)}}
\begin{definition}[$1$-Layer Transformer]
	\label{def:one_layer_transformer}
        For a single layer model, $\Model$,
	\begin{align*}
		\Model &= \Unembed \circ \left(\left( \attn^\ly{i} + \MLP^\ly{i} \right) \circ \LayerNorm + \Iden \right) \circ \Embed
	\end{align*}
	where $\attn$ has heads $\attnH_1, \dots, \attnH_H$ for some number of attention heads $H$.
\end{definition}

We also introduce notation related to one-hot vectors:
\begin{definition}[One Hot Space]
	\label{def:one_hot_space}
	Let $\dVocab$ be the size of the vocabulary.
	We define $\OneHotSpace \subset \R^\dVocab$ as the set of one-hot vectors.
	That is, $\OneHotSpace = \{ \vec{e}_i \mid i \in [\dVocab] \}$ where $\vec{e}_i$ is the $i$-th standard basis vector in $\R^{\dVocab}$.
\end{definition}

%
%

%% file: sections/framework.tex
\section{Worst-Case Deviation and a Custom Norm}
\label{sec:worst_case_framework}
When proving statements about our model, we will find it useful to talk about a variation on the Lipschitz constant and maximum absolute deviation that we term the ``worst-case deviation.''
Simply put, the worst-case deviation of a function $f: \calX \to \R^n$ is the maximum distance between the outputs of $f$ on any two points in $\calX$.

Before defining the worst-case deviation more formally, we will first define the Lipschitz constant of a function.
\begin{definition}[Lipschitz Constant]
	\label{def:lipschitz_constant}
	For a function $f: \calX \to \R^n$ and norm $p$, we define the Lipschitz constant of $f$ as
	\[
		\Lip(f)_p := \sup_{X_1, X_2 \in \calX} \frac{\norm{f(X_1) - f(X_2)}_p}{\norm{X_1 - X_2}_p}.
	\]
\end{definition}

Then, we can define the worst-case deviation as the follows.
\begin{definition}[Worst-Case Deviation]
	\label{def:worst_case_deviation}
	For a function $f: \mathcal{X} \to \R^n$, we define the worst-case deviation of $f$ as
	\[
		\WD(f; \calX)_p := \sup_{X_1, X_2 \in \calX} \norm{f(X_1) - f(X_2)}_p
	\]
\end{definition}


In \cref{sec:proofs_worst_case_deviation}, we state and prove some useful properties of the worst-case deviation.

When discussing the worst-case deviation in this paper, we will use the $\infty$-norm for vectors.
For convenience, we will define an analogue of the Frobenius norm for the $\infty$-norm.

\begin{definition}[$\frobInf$-norm]
	\label{def:frobenius_infinity_norm}
	For a matrix $M \in \R^{n \times m}$, we define the Frobenius-$\infty$ norm ($\frobInf$-norm) as
	\[
		\norm{M}_{\frobInf} := \max_{i \in [n]} \max_{j \in [m]} \abs{M_{ij}}.
	\]
\end{definition}

%% file: sections/technical_precurse.tex
\newcommand{\ModelFinal}{\Model^{final}}
\newcommand{\peakToPeak}{\mathrm{PTP}}
\newcommand{\RepIndxs}{\mathcal{R}}
\newcommand{\FreeIndxs}{\mathcal{F}}
\newcommand{\softSumMaxF}{\alpha_{free}^{\max}}
\newcommand{\softSumMinR}{\alpha_{fix}^{\min}}
\newcommand{\softSumMaxR}{\alpha_{fix}^{\max}}
\newcommand{\softSumMinF}{\alpha_{free}^{\min}}

\newcommand{\ssMaxFA}{{\alpha}_{free}^{\max}}
\newcommand{\ssMinRA}{{\alpha}_{fix}^{\min}}
\newcommand{\ssMaxRA}{{\alpha}_{fix}^{\max}}
\newcommand{\ssMinFA}{{\alpha}_{free}^{\min}}

\newcommand{\ssMaxFB}{{\beta}_{free}^{\max}}
\newcommand{\ssMinRB}{{\beta}_{fix}^{\min}}
\newcommand{\ssMaxRB}{{\beta}_{fix}^{\max}}
\newcommand{\ssMinFB}{{\beta}_{free}^{\min}}

\newcommand{\logRMin}{L_{fix}^{\min}}
\newcommand{\logRMax}{L_{fix}^{\max}}
\newcommand{\logFMin}{L_{free}^{\min}}
\newcommand{\logFMax}{L_{free}^{\max}}
\newcommand{\Samp}{\textsf{Samp}}
\newcommand{\varAt}{{\Var_j^{(k)}}}
\newcommand{\varAtMin}{{\Var_{j, \min}^{(k)}}}
\newcommand{\varAtMax}{{\Var_{j, \max}^{(k)}}}
\newcommand{\preSMMax}{\ell_{i, \max}^{(k)}}
\newcommand{\preSMMin}{\ell_{i, \min}^{(k)}}

\section{Input Restrictions and Proving Overwhelming}
\label{sec:meta_framework}

In this section, we will provide an algorithm to decide ``overwhelming."
Along the way, we develop a generalizable method of upper-bounding the worst-case deviation of a single layer of a transformer model under input restriction.

\subsubsection*{Input Restrictions and Designed Space}
We use the notation from \textit{Analysis of Boolean Functions} to denote restrictions on inputs to a function \cite{o2014analysis}.
The restriction will fix certain tokens to be a specific value and leave the rest of the tokens free.

\begin{definition}[Input Restriction, \cite{o2014analysis} definition 3.18]
	\label{def:input_restriction}
	Let $f : \calX^n \rightarrow \calY$ be some function and $J \subset [n]$ and $\notJ = [n] \setminus J$.
	Let $z \in \calX^\notJ$.
	Then, we write $f_{J \mid z} : \calX^J \rightarrow \calY$ (``the restriction of $f$ to $J$ given $z$'') as the subfunction of $f$ that is obtained by fixing the coordinates of $\notJ$ to the values in $z$.
	Given $y \in \calX^J$ and $z \in \calX^\notJ$, we write $f_{J \mid z}(y)$ as $f(y, z)$ even though $y$ and $z$ are not literally concatenated.
\end{definition}

Throughout this paper, we will consider a specific input restriction where the first $\nfix$ tokens are restricted to a string $\desSet$ and the last token is fixed to token $\query$.
We will denote said restriction as $\desF{f}$ where $(\nfix, \nctx)$ denotes the set $\{ \nfix + 1, \ldots, \nctx - 1\}$.

Next, it will be useful to define the set of all possible inputs under an input restriction.

\begin{definition}[Designed Space]
	\label{def:InpSpace}
	Recall, from \cref{def:one_hot_space}, that $\OneHotSpace$ is the set of all one-hot vectors of size $\dVocab$.
	We denote by $\OneHotSpace^n$ the set of matrices where each row is a one-hot vector of size $\dVocab$.
	Then, let $\InpSpace \subset \OneHotSpace^\nctx$ designate the set of all possible inputs under a \textbf{specific} input restriction.
	That is, 
	\[
		\InpSpace = \left\{ X \in \R^{\nctx \times \dVocab} \mid X = 
		\begin{bmatrix}
			\vece_{\desSet_1} \\
			\vece_{\desSet_2} \\
			\vdots \\
			\vece_{\desSet_s} \\
			Y \\
			\vecQuery
		\end{bmatrix}, Y \in \freeSpace
		\right\}.
	\]
	where $\freeSpace$ is the space of free tokens.
\end{definition}
\subsection{Algorithm for deciding Overwhelming}
Here we consider zero temperature sampling setting where we can define our model with sampling as
\[
	\ModelFinal(X) = \arg\max_{i \in [\dVocab]} \Model(X) \cdot \vec{e}_i^T
\]
where $\Model$ is defined in \cref{def:one_layer_transformer}.
The model simply selects the token with the highest logit weight rather than sampling from the output distribution. 

We define the ``peak-to-peak difference'' to be the difference between the logit for the most likely token and the logit for the second most likely token for sample $X$.

\begin{definition}[Peak-to-peak difference]
    Let $X \in \InpSpace$ be any element from the restriction and $k = \ModelFinal(X)$.
    Then, we let 
    \[
    	\peakToPeak(\Model, X) = \min_{j \in [\dVocab], j \neq k} \Model(X) \cdot \left(\vec{e}_{k}^T - \vec{e}_j^T\right).
    \]
\end{definition}

To prove a model is overwhelmed by a fixed input we bound the worst-case deviation by the peak-to-peak difference.
The following theorem summarizes the bound our algorithm is tasked with verifying. 
\begin{theorem} \label{thm:metathm}
        If
        \[
		\WD(\Model; \InpSpace)_\infty < \peakToPeak(\Model, X) / 2
	\]
        for some $X \in \InpSpace$,
	then the output of $\ModelFinal$ is ``overwhelmed'' under the restriction.
\end{theorem}
\begin{proof}
	As $\ModelFinal$ always selects the token with the maximum logit value, if the coordinate-wise deviation of the restricted model's output never differs by more than $\peakToPeak(\desF{\Model}, X) / 2$ for any $X \in \InpSpace$, the token with the second highest logit value will never exceed that of the token with the highest value.
\end{proof}

Ultimately, we will want to make use of the above theorem alongside \cref{alg:overwhelmCheckDet} to prove the following theorem:
\begin{theorem}[Input Restriction] \label{thm:InpRes}
	If
    \begin{align*}
	    &\WD(\desF{\Model}; \InpSpace)_\infty
        \\&< \peakToPeak(\desF{\Model}, X) / 2,
    \end{align*}
	then the output of model $\desF{\Model}$ is fixed for all inputs in $\InpSpace$.
	Moreover, we can use \cref{alg:overwhelmCheckDet} to produce an upper bound $W$ for $\WD(\desF{\Model}; \InpSpace)_\infty$.
\end{theorem}

To prove the above theorem, we will:
\begin{itemize}[nosep]
    \item break down the model into its components.
    \item use the triangle inequality of worst-case deviation to bound the worst-case deviation of the model prior to unembedding.
    \item use the Lipschitz constant of the unembedding matrix to bound the worst-case deviation of the model.
\end{itemize}
The formal proof can also be found in \cref{subsec:InpResProof}.

\begin{algorithm}[tb] 
	\caption{Algorithm for deciding Overwhelming}
	\label{alg:overwhelmCheckDet}
	\begin{algorithmic}

		\STATE {\bfseries Input:}  Model $\Model$, fixed string $s$, contenxt length $\nctx$, query token $q$.
		\STATE {\bfseries Output:} ``Overwhelmed'' or ``Inconclusive''.
		\STATE
            \STATE Calculate pre-softmax extremal logit values $\ell^{\min}$ and $\ell^{\max}$ via \cref{alg:lminlmax}.
		\STATE Using the above, calculate $\ssMaxRB$ and $\ssMinRB$ as in \cref{def:soft-extrem-values} and \cref{lem:min_max_softmax}.
		\STATE Calculate upper-bound $W^{\attn}$ as in \cref{lem:att_bound}.
		\STATE Set $W = W^{\attn}		$
		as per \cref{lem:mlpbound} and 
		\STATE Sample $X \gets \InpSpace$
		\IF{$W < \peakToPeak(\desF{\Model}, X) / 2$ }
		\STATE Return ``Overwhelmed''
		\ELSE
		\STATE Return ``Inconclusive''
		\ENDIF
	\end{algorithmic}
\end{algorithm}

\subsection{Proof Overview of \cref{thm:InpRes} (Algorithm Correctness)} \label{sec:algoCorrectness}
The proofs for all the lemmas and statements in this section are deferred to \cref{sec:proofs_framework}.

\subsubsection*{Breaking Down the Model}
In a single-layer model, layer-norm turns out to be quite simple and almost trivial.
Simply put, we can take the layer-norm function and re-write it as a change of basis from the basis specified by the embedding matrix, $\Embed$, to a basis specified by a ``normalized'' embedding matrix.

\begin{definition}[Layernorm Matrix, $\EmbedLN$]
	\label{def:blowup_shift}
	For row $i \in [\dVocab]$ of the embedding matrix, $E[i] \in \R^\dEmb$, define the normalized vector $\vec{n}_i$ as
	$$
	\vec{n}_i = \LayerNorm(E[i]) = \frac{E[i] -\E[E[i]] \cdot \OneVec}{\sqrt{\Var[E[i]] + \eps}} \cdot \gamma + \beta \cdot \OneVec.
	$$
	Then, define the matrix $\EmbedLN \in \R^{\dVocab \times \dEmb}$ as $\EmbedLN[i] = \vec{n}_i$.
\end{definition}

We will now break down the model in a way such that it is easier to analyze.
For component $\component \in \{\attn, \MLP, \Iden\}$, we will use $f^{\component} : \OneHotSpace^\nctx \to \R^{\nctx \times \dVocab}$ to denote the following functions:
\begin{align*}
	f^{\Iden}(X) &= \Unembed \circ \Iden \circ  E \cdot X, \\
	f^{\MLP}(X) &= \Unembed \circ \MLP \circ  \EmbedLN \cdot X, \\
	f^{\attn}(X) &= \Unembed \circ \attn \circ  \EmbedLN \cdot X.
\end{align*}

We can rewrite our model as
\begin{align*}
	\Model &= f^{\attn} + f^{\MLP} + f^{\Iden}
\end{align*}
and so, by the triangle inequality for worst-case deviation,
\begin{align*}    
    &\WD(\desF{\Model}; \InpSpace) \\
    \leq& \sum_{\component} \WD(\cdot \vec{e}_\nctx \cdot \desF{f^{\component}}; \InpSpace)
\end{align*}
for $\component \in \{\attnH, \MLP, \Iden\}$.

\subsubsection{Bounding Attention}
\label{subsec:attention_bounds}
For a multi-headed attention mechanism, $\attn(X) = [\attnH_1(X); \ldots; \attnH_H(X)]$: i.e.\ the output is a concatenation of the individual attention head outputs.
So, the infinity norm worst-case deviation is simply bounded by the maximum worst-case deviation over the individual attention heads.
We thus have the following lemma:
\begin{lemma}
	\label{lem:att_bound_by_heads}
	\begin{align*}
		&\WD(\vec{e}_\nctx \cdot \desF{f^{\attn}} ; \InpSpace)_\infty 
	     \\ &\leq 
	     \max_h \WD(\vec{e}_\nctx \cdot \desF{f^{\attnH_h}} ; \InpSpace)_\infty.
	\end{align*}
\end{lemma}
The rest of this section will focus on bounding the worst-case deviation of a single attention head which will be the most challenging part of the proof.

We will show how to bound the contribution of each token to the attention weights by:
(1) establishing bounds for the attention weight coming from fixed tokens to the query token, (2) establishing an upper bound for the attention weight coming from free tokens to the query token.
We can then demonstrate that the attention weights are heavily biased towards the fixed tokens.

We implicitly consider rotary-type positional encodings in the attention mechanism, where the positional encodings are absorbed into the calculation of the logits.
As previously mentioned, we use $\ell$ to denote the logits on the query token.  It will be useful, in the proof of our results, to define and compute worst-case upper and lower bounds for $\ell$, which we define formally in \cref{def:lminlmax}.
To obtain bounds $\ell^{\min}$ and $\ell^{\max}$, we provide a simple algorithm in \cref{alg:lminlmax} to compute a bound on the logits:
\begin{lemma}[Bounds on Logits]
	\label{lem:lminlmax}
	Algorithm~\ref{alg:lminlmax} computes the minimum and maximum values of the logits $\ell$. 
\end{lemma}


We now define the most consequential value to bound the worst-case deviation of the attention mechanism.
\begin{definition}[Softmax Extremal Values]\label{def:soft-extrem-values}
	Let $\softSumMinF = \sum_{j \in (\nfix, \nctx)} e^{\ell_j^{\min}}$ and $\softSumMaxF = \sum_{j \in (\nfix, \nctx)} e^{\ell_j^{\max}}$.
	Similarly, let $\softSumMinR = \sum_{i \in [\nfix] \cup \{\nctx\}} e^{\ell_i^{\min}}$ and $\softSumMaxR = \sum_{i \in [\nfix] \cup \{\nctx\}} e^{\ell_i^{\max}}$.
	Then, 
	\[
		\ssMinFB = \frac{\softSumMinF}{\softSumMaxR + \softSumMinF} \quad \text{and} \quad \ssMaxFB = \frac{\softSumMaxF}{\softSumMinR + \softSumMaxF}
	\]
	and
	\[
		\ssMinRB = \frac{\softSumMinR}{\softSumMinR + \softSumMaxF} \quad \text{and} \quad \ssMaxRB = \frac{\softSumMaxR}{\softSumMaxR + \softSumMinF}.
	\]
\end{definition}

These extremal values can be used to get the bounds:

\begin{lemma}[Minimum and Maximum after Softmax]
	\label{lem:min_max_softmax}
	\begin{equation}
		\label{eq:softmax_upper}
		\ssMinFB \leq
		\sum_{j \in (\nfix, \nctx)} \softmax(\ell_j) \leq
		\ssMaxFB
	\end{equation}
	aswell as,
	\begin{align}
		\label{eq:softmax_lower}
		\ssMinRB \leq
		\sum_{j \in [\nfix] \cup \{\nctx\}} \softmax(\ell_j) \leq
		\ssMaxRB.
	\end{align}
\end{lemma}

In the following, let $\EmbedLN[[\nfix] \cup \{\nctx\}, :] \in \R^{\nfix \times \dEmb}$ denotes the matrix with the rows in $[\nfix] \cup \{\nctx\}$ selected. 

\begin{lemma}[Worst-case Deviation of Attention]
	\label{lem:att_bound}
	Worst-case deviation of attention is bounded as follows,
	\begin{align*}
		&\WD(\vecNctx \cdot f^{\attnH}_{\ndes \mid r, q} ; \InpSpace)_\infty \leq
	     (\ssMaxRB - \ssMinRB) \\
         & \cdot \norm{\EmbedLN[[\nfix] \cup \{\nctx\}, :] \cdot V \cdot \Unembed}_{\frobInf} \\
	 &+ 2 \cdot \ssMaxFB \cdot \norm{\EmbedLN \cdot V \cdot \Unembed}_{\frobInf}.
	\end{align*}
        where $V$ is the value matrix in the attention head (see \cref{sec:appendix_model}).
\end{lemma}



%

\subsubsection{Bounding MLP and Identity}
Because the MLP and identity components are simple feed-forward networks without any ``cross-talk'' between tokens and the last token of the model is fixed within our scheme, we can trivially see that the worst-case deviation of these components are $0$.

\begin{lemma}[Worst-case Deviation of MLP and Identity]
	\label{lem:mlpbound}
	\begin{align*}
		\WD(\vec{e}_\nctx \cdot \desF{f^{\MLP}} ; \InpSpace)_\infty  = 0
	\end{align*}
	and
	\begin{align*}
		\WD(\vec{e}_\nctx \cdot \desF{f^{\Iden}}  ; \InpSpace)_\infty = 0.
	\end{align*}
\end{lemma}


\subsubsection{Causal Masking in Attention}
We note that a causal mask will not change the attention scores for the last token in a sequence (i.e.\ the query token).
And so, the attention scores for the query token will be the same regardless of whether we use a causal mask or not.
Our outlined algorithm (\cref{alg:overwhelmCheckDet}) thus still provides a valid upper bound on the worst-case deviation for models with and without causal masking.


%% file: sections/model_collapse.tex


\section{Overwhelming Under Permutation Invariance}
\label{sec:perm_invar}
In \cref{sec:meta_framework}, we provided a concrete algorithm that decides overwhelming for a fixed context size $\nctx$.
In this section, we provide a more refined algorithm which can provide a tighter bound on the worst-case deviation in the restricted setting where the free tokens are allowed to be any permutation of a fixed string.
We introduce a new technique to achieve tighter bounds: specifically, we use an integer-linear program relaxation to model the ``global'' constraint on the free tokens.

\begin{definition}[Permutation Equivalence relation]
	We will define the equivalence relation \(\sim\) on \(\InpSpace\).
	Let $X = (\desSet || X_f || \vecQuery) \in \InpSpace$ where $X$ are the free tokens, $\desSet$ are the fixed tokens, and $\vecQuery$ is the query token.
	Then, we define \(\sim\) such that for $X = (\desSet || X_f || \vecQuery), Y = (\desSet || Y_f || \vecQuery) \in \InpSpace$,
	\[
		X \sim Y \iff X_f = \pi(Y_f)
	\]
	where $\pi$ is a permutation of the free tokens.
	Note that if $X \sim Y$, then $\E[X \cdot \Embed] = \E[Y \cdot \Embed]$ and $\Var[X \cdot \Embed] = \Var[Y \cdot \Embed]$.
    Moreover, let $\freeToks$ be the multi-set of potential free tokens for $[X]$.
    I.e. $\freeToks = \{X_f[1], \dots, X_f[\nfree] \}$.
\end{definition}

To bound the worst-case deviation over an equivalence class we use a similar algorithm (\cref{alg:overwhelmCheck2}) to \cref{alg:overwhelmCheckDet}.
We still bound the worst-case deviation of attention, but now the set of free tokens are restricted, leading to a tighter bound.


\begin{theorem}[Input Restriction and Permutation Invariance] \label{thm:InpResPermInv}
	If
	\[
		\WD(\desF{\Model}; [X])_\infty < \peakToPeak(\desF{\Model}, X) / 2
	\]
	then the output of model $\desF{\Model}$ is fixed for all inputs in $[X]$.
	Moreover, \cref{alg:overwhelmCheck2} produces an upper bound $W$ for $\WD(\desF{\Model}; \InpSpace)_\infty$.
\end{theorem}
A full proof of \cref{thm:InpResPermInv} is provided in \cref{subsec:proofInpResPermInv}.

\begin{algorithm}[h] 
	\caption{Algorithm for deciding Overwhelming}
	\label{alg:overwhelmCheck2}
	\begin{algorithmic}
		\STATE \textbf{Input:} $\Model, \desSet, \query$
		\STATE \textbf{Output:} Return ``Overwhelemed'' or ``Inconclusive''
		\STATE
		\STATE Set $\fenc(\vec{e}) = \vece \cdot \EmbedLN$
		\STATE Calculate $\ssMinFA, \ssMaxFA$ using the naive algorithm (\cref{alg:naiveAlpha}) or the linear program (\cref{alg:LPAlpha})
		\STATE Let
		\begin{align*}
			\alpha_{fix} = 
			\exp\left(
				\fenc(\vec{e}_q) \cdot Q \Theta_{\nctx, \nctx} K^T \fenc(\vec{e}_{q})
			\right)    \\
			+  \sum_{j \in [\nfix]} \exp\left(
				\fenc(\vec{e}_q) \cdot Q \Theta_{i, \nctx} K^T \fenc(\vec{e}_{\desSet_j})
			\right)
		\end{align*}
		\STATE Calculate $\ssMaxFB$, $\ssMinRB$, $\ssMinRB$
		 \STATE 
		 Calculate $W^{\attn}$ as in \cref{lem:attn_perm}
		\STATE Set $W = W^{\attn} $ 
		\IF{
		$W <  \peakToPeak(\desF{\Model}, X) / 2$ }
		\STATE Return ``Overwhelmed''
		\ELSE 
		\STATE Return ``Inconclusive''
		\ENDIF
	\end{algorithmic}
\end{algorithm}

\subsection{Bounding Attention}
To bound the worst-case deviation of attention, we provide two algorithms to bound $\ssMinFA, \ssMaxFA$ as defined in \cref{lem:min_max_softmax}.
The first, in \cref{alg:naiveAlpha}, takes a ``naive'' approach by iterating through the position of the free tokens and finding an extremal value for each position.
The second, in \cref{alg:LPAlpha}, uses a linear program to get a tighter bound.
To see why \cref{alg:LPAlpha} provides a valid bound, consider a similar program except where an \emph{integer} linear program is used.
Then, the integer version of \cref{alg:LPAlpha} finds the permutation of free tokens which maximizes (resp. minimizes) the pre-softmax logits.
Relaxing to a linear program, we have an upper-bound (resp. lower-bound) on the pre-softmax logits.
We can formally state the correctness of \cref{alg:naiveAlpha} and \cref{alg:LPAlpha} in the following lemma:
\begin{lemma}[Correctness of \cref{alg:LPAlpha} and \cref{alg:naiveAlpha}]
	\label{lem:corrAlpha}
	 Both \cref{alg:LPAlpha} and \cref{alg:naiveAlpha} provide valid bounds on $\ssMinFA$ and $\ssMaxFA$.
\end{lemma}
\vspace{-0.2cm}

\begin{algorithm}[tb]
	\caption{
		Linear Program to find $\ssMinFA$.\\
		To find $\ssMaxFA$, switch the $\min$ to a $\max$ in the objective.
	}
	\label{alg:LPAlpha}
	\begin{algorithmic}
	\STATE \textbf{Input:} {$\fenc$ and model $\Model$}
	\STATE
	\STATE Return the optimal value to the following linear program:
	\begin{align*}
		\min \quad & \sum_{j \in (\ndes, \nctx)} \exp\bigg(\sum_{t \in \freeToks} \\
			\fenc(\vecQuery) \cdot &Q \cdot \PosRot_{j, \nctx} K^T \cdot \fenc(\vec{e}_t)^T
	\bigg) \cdot x_{j, t} \\
		\text{subject to} \quad
			   & \forall j, \sum_{t \in \freeToks} x_{j, t} = 1 \\
			   & \forall t, \sum_{j \in (\ndes, \nctx)} x_{j, t} = 1, \\
			   &  \forall j, t, x_{j, t} \geq 0.
	\end{align*}
	\end{algorithmic}
\end{algorithm} 

Then, with the bounds on the softmax contributions, we can state a specialized version of the bound on attention's worst-case deviation for the permutation invariant case.

Then, we can state the following lemma:
\begin{lemma}[Worst-case Deviation of Attention]
	\label{lem:attn_perm}
	Worst-case deviation of attention is bounded as follows when considering the permutation class $[X]$ with free tokens $\freeToks$:
	\begin{align*}
		&\WD(\vecNctx \cdot f^{\attnH}_{\ndes \mid r, q} ; \InpSpace)_\infty
	     \\ &\leq
		(\ssMaxRB - \ssMinRB) \cdot \|\EmbedLN[[\nfix] \cup \{\nctx\}, :] \cdot V  \\
		& \cdot \Unembed\|_{\frobInf} \\
		&+ 2 \cdot \ssMaxFB \cdot \norm{\EmbedLN[(\nfix,\nctx),: ] \cdot V \cdot \Unembed}_{\frobInf}.
	\end{align*}
\end{lemma}
The proof follows analogously to that of \cref{lem:att_bound} except for two main differences:
(1) we do not necessarily use all of the free tokens in the dictionary and can thus restrict out search, and
(2) we have some global structure on the free tokens and thus can use a linear program to find the extremal softmax contributions.

After which the proof of \cref{thm:InpResPermInv} follows analogously to that of \cref{thm:InpRes}.

%% file: sections/model_eval.tex
\section{Evaluating on a Single Layer Model}
\label{sec:model_eval}

In this section, we empirically demonstrate our algorithm on a single layer transformer model trained for next token prediction on a standard text corpus.
Specifically, we run \cref{alg:overwhelmCheck2} to calculate a bound on the worst-case deviation and peak-to-peak difference for the model where the domain has an input restriction with free tokens drawn from a single permutation class as in \cref{sec:perm_invar}. When \cref{alg:overwhelmCheck2} outputs $\OverwQ$ it follows from \cref{thm:InpResPermInv} (informally, \cref{thm:informal_overwhelm_perm} in the Introduction) that the output of the model is the same for all permutations.

\subsection{Experimental Setup}
We train a single-layer transformer model with causal masking, embedding dimension of 512, and the BERT tokenizer \cite{devlin-etal-2019-bert}.
The model's architecture is outlined in \cref{sec:model}.
We train on the AG News dataset \footnote{see \href{https://pytorch.org/text/stable/_modules/torchtext/datasets/ag_news.html}{PyTorch's documentation}}, using the training split with a batch size of 8, learning rate of 5e-5 using the Adam optimizer, and $20$ epochs.

We examine a few different types of input restrictions
\footnote{
    The first token is always fixed to the ``BOS'' token which connotes the beginning of a string in the BERT tokenizer. For simplicity, we think of the BOS token as part of the fixed tokens.
}:
\begin{itemize}[nosep]
	\item \textbf{Random String:} Randomly sample alpha-numeric tokens (including space and punctuation) to form $\desSet$
	\item \textbf{Random Sentences:} Sample (and concatenate) sentences from the AG News testing set, to form $\desSet$
	\item \textbf{Repeating Tokens:} Form $\desSet$ by repeating the string ``what is it'' (which is $3$ tokens in the BERT tokenizer).
\end{itemize}

In each case, we have a \emph{tunable} number of tokens in $\desSet$\footnote{Larger input restrictions are produced by concatenation.}.
For the permutation class, we consider the following two cases:

\begin{itemize}[nosep]
	\item \textbf{The Old Man and the Sea}: We fix the free tokens to be a snippet from the opening chapter of Hemingway's \emph{The Old Man and the Sea}:
		 ``The blotches ran well down the sides of his face''
	\item \textbf{Hamlet}: 
		We use the famous quote from Shakespeare's \emph{Hamlet}:
    ``To be, or not to be, that is the question: Whether 'tis nobler in the mind to suffer The slings and arrows of outrageous fortune,''

\end{itemize}

Finally, the question mark token ``?'' is used as the query token.

In \cref{sec:appOut}, we plot the worst-case deviation in \cref{fig:worst-case} and the peak-to-peak difference in \cref{fig:ptp} for each of the input restrictions and permutation classes.


\subsection*{Observations}
The experimental results for worst-case deviation bounds in \cref{fig:worst-case} and peak-to-peak difference in \cref{fig:ptp} provide interesting insights and questions.
The upper-bound on worst-case deviation trend downwards as the fixed input token size is increased.
But only the ``repeated-string'' setting is monotonically decreasing.
We note that the linear program seems to give the most advantage over the naive approach for the ``Hamlet'' free string.
This makes sense, as the ``Hamlet'' string is much longer than the ``Old Man and the Sea'' string, and thus the linear program has more room to give a better bound.

\subsection{Overwhelming Continues Through Generation}
When a transformer model is \overwQ the generated token immediately after the text is fixed for any free string. However, this does not apriori imply that subsequent tokens when the model is used to repeatedly generate tokens will be fixed. 

To test this we run examples of text generation where we begin with an overwhelming string $s$ and use the model to generate text. We then check if the model remains \overwQ as the generated text is included as part of the fixed string. The results of these tests are plotted in \cref{fig:contOverwhelm} within \cref{sec:appOut}. One would hope the model remains overwhelmed as the fixed string is increased by the addition of the newly generated tokens. 
This seems to occur for both strings: once overwhelming provably occurs, the model remains \overwQ for the remainder of the generated tokens.

%% file: sections/conclusion.tex
\section{Conclusion and Further Work}
\label{sec:conclusion}
In this work, we introduce the notion of ``overwhelming" transformer models.
We provide algorithms to provably check if a trained one-layer transformer is overwhelmed by some fixed input string $s$ and query token $q$.
We then empirically run the algorithm on a concrete single-layer model.
We obtain bounds and find examples of natural overwhelming strings for this model in the restricted permutation setting.

This work is a first step in building algorithms to give provable guarantees for practical models.
Improving our existing bounds and extending the algorithms to multi-layer transformers are important problems in this direction.
Moreover, we believe that our approach to proving overwhelming, the use of LP bounds, and worst-case deviation metric will be independently useful for other theoretical studies of transformer based models.

\subsection*{Acknowledgment}
The authors are grateful for helpful discussions and feedback from Jason Gross and Peter Bajcsy.
LS acknowledges funding from the NSF Graduate Research Fellowship.

%% file: sections/appendix_plots.tex
\section{Experimental Outcomes}
\label{sec:appOut}
In this appendix we provide the plots for \cref{sec:model_eval}.
Firstly, \cref{fig:worst-case} contains the plots of worst-case deviation with and without the use of the linear program, and  \cref{fig:ptp} plots the peak-to-peak difference. In \cref{fig:ptp} whenever our algorithm proves ``overwhelming'', the point is colored red and marked with a `x.'
\vspace{0pt}\nopagebreak
\begin{figure}[H]
    \centering
	\includegraphics[width=\textwidth, keepaspectratio, trim={0cm 0cm 0 1.5cm},clip]{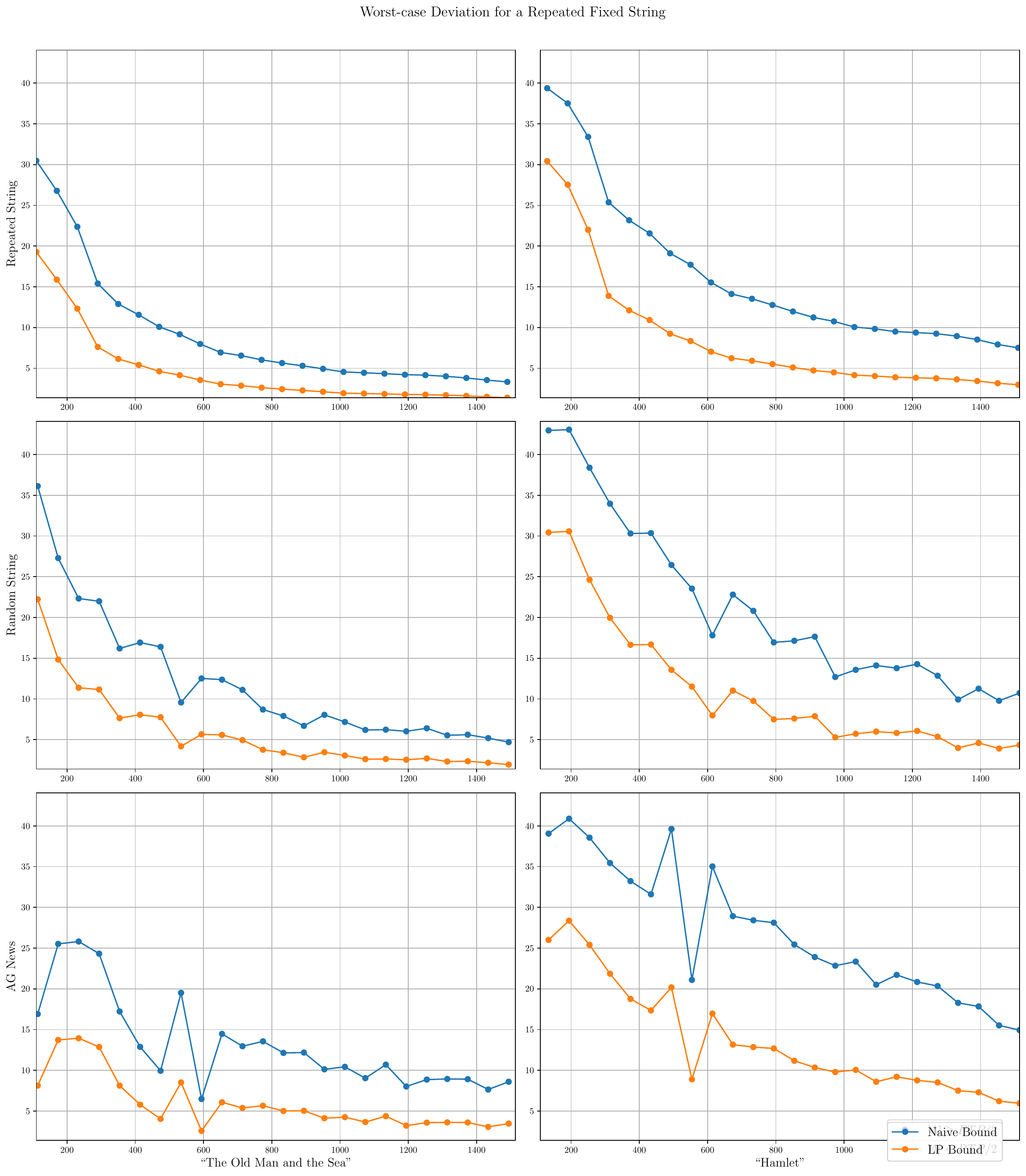}
	\caption{The worst-case deviation for the model for the three different input restrictions and two different permutation classes.
    The $x$-axis is the number of tokens and the $y$-axis is the upper-bound on worst-case deviation.
    }
	\label{fig:worst-case}
\end{figure}
\pagebreak
\begin{figure}[H]
    \centering
	\includegraphics[width=\textwidth, keepaspectratio, trim={0cm 0cm 0 1.5cm},clip]{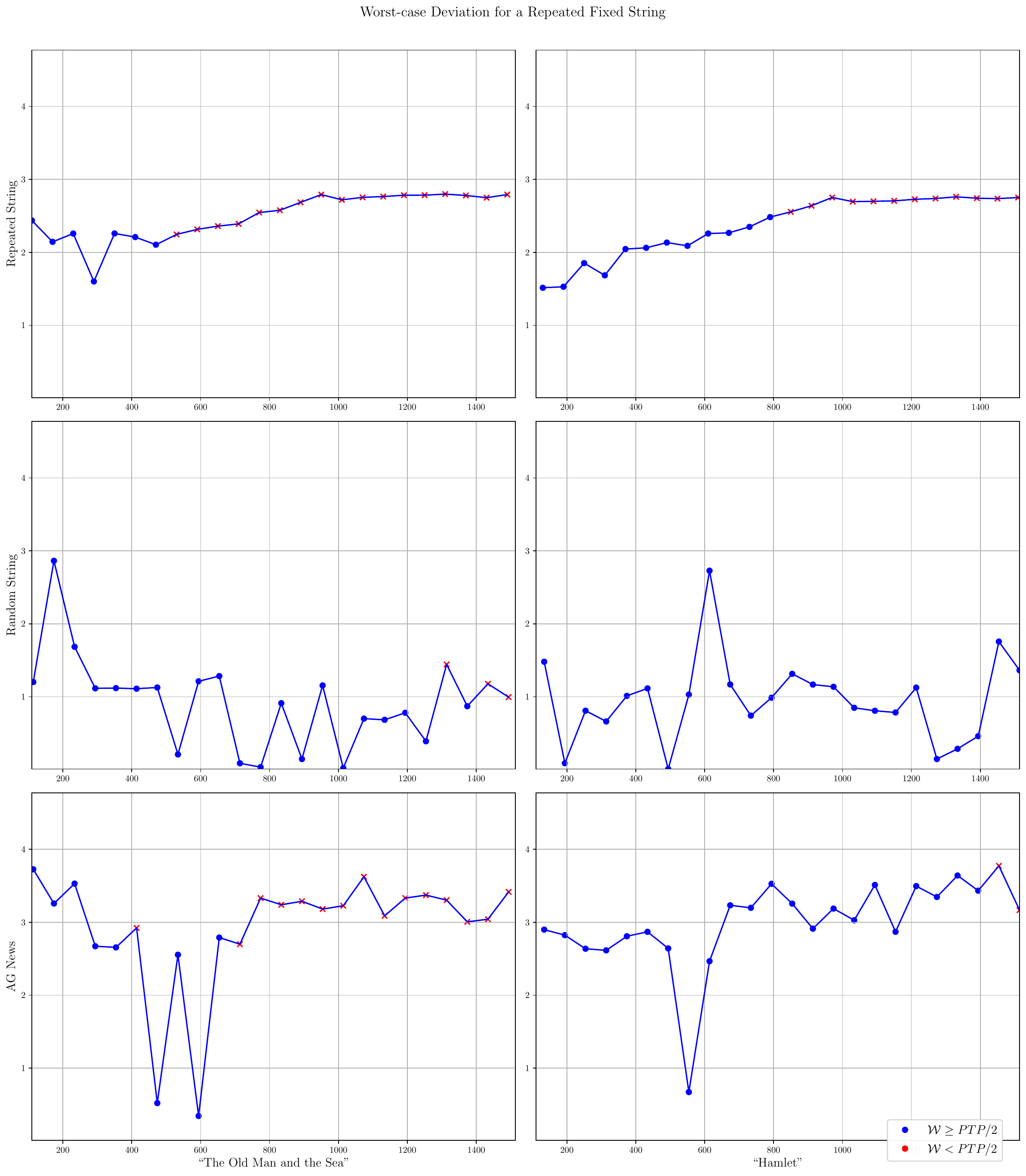}
	\caption{The peak-to-peak difference for the model for the three different input restrictions and two different permutation classes.
	Red ``x''s indicate that the worst-case deviation is less than half of the peak-to-peak difference and, thus, the model output is provably invariant over the permutation class.
        The $x$-axis is the number of tokens and the $y$-axis is the peak-to-peak deviation.
    }
	\label{fig:ptp}
\end{figure}

\subsection*{Overwhelming Through Generation}
In \cref{fig:contOverwhelm} we plot the bounds on worst-case deviation and the peak-to-peak difference as the model is used to continually generate text. The newly generated tokens are continuously added to the fixed string and fed again to the model to generate the next token. The model is ``overwhelmed'' whenever the worst-case deviation is less than the peak-to-peak difference in the plot.

\vspace{-20pt}\nopagebreak
\begin{figure}[H]
    \centering
	\includegraphics[ width=\textwidth, keepaspectratio, trim={0cm 0cm 0 1.5cm},clip]{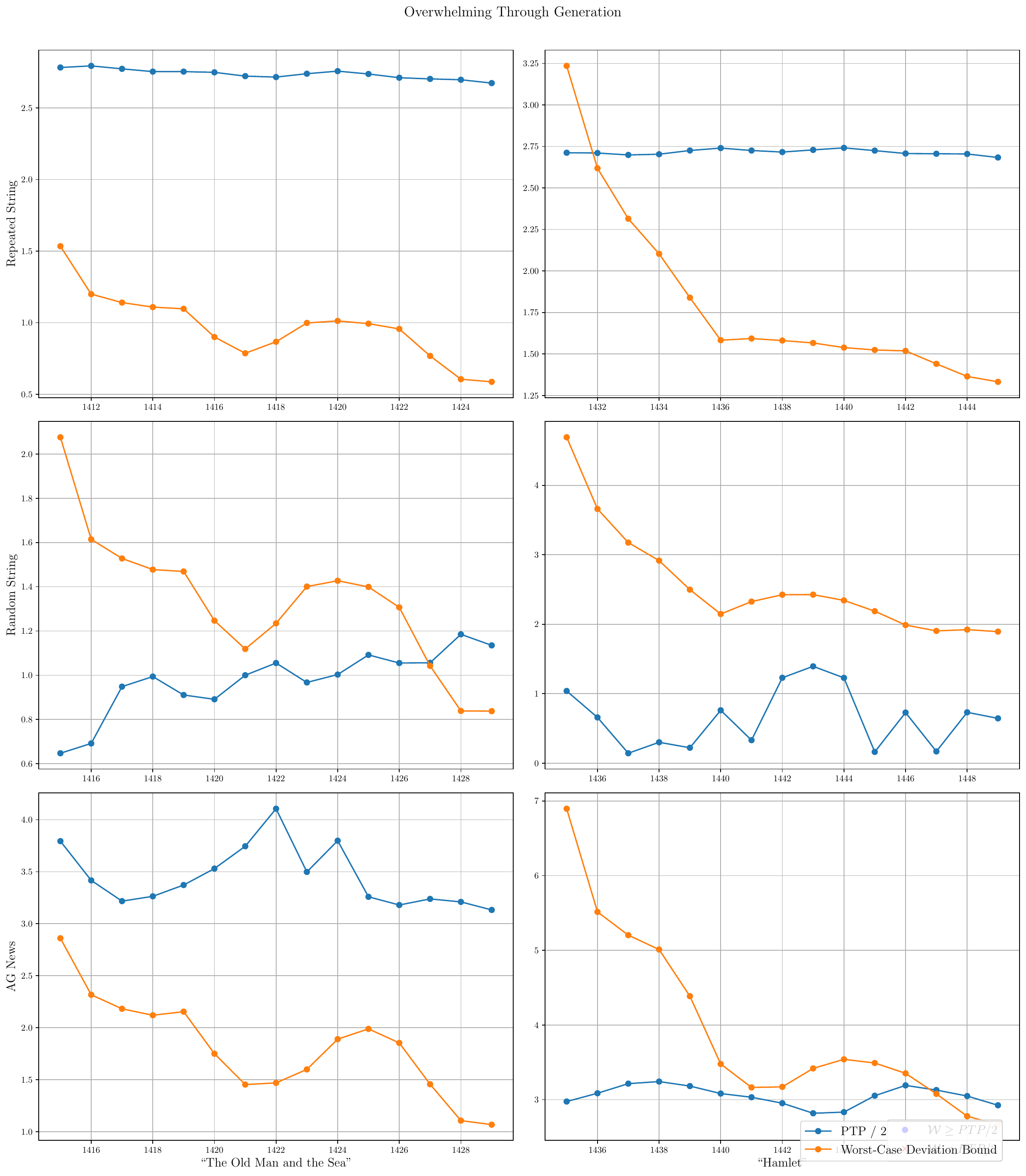}
	\caption{
            The peak-to-peak difference and the worst-case deviation when we continue generation through multiple tokens.
            The $x$-axis connotes the total number of tokens ($n_{ctx}$) throughout a generation.
            When the line for $PTP / 2$ is above the worst-case deviation line, then our algorithm provably guarantees that the output is fixed.}
	\label{fig:contOverwhelm}
\end{figure}

%% file: sections/appendix_model.tex
\section{Appendix for Model Details}
\label{sec:appendix_model}
We provide a formal definition of each of the components in the model used in this paper.
\begin{itemize}
	\item $\softmax$ is the softmax function
		\[
			\softmax(\vec{\alpha})[i] = \frac{e^{\vec{\alpha[i]}}}{\sum_{i=1}^{\dVocab} e^{\vec{\alpha}[i]}}
		\]
	\item $\relu$ is the rectified linear unit function
		\[
			\relu(x) = \max(0, x)
		\]
	\item $\LayerNorm$ is the layer normalization function which for matrix $X$ performs the following row-wise (token-wise) function for rows of $X$ (i.e.\ $X[i, :]$):
		\[
			\LayerNorm(X[i, j]) = \frac{X[i, j] - \Expec[X[i, :]]}{\sqrt{\Var[X[i, :]] + \eps} } \cdot \gamma + \beta
		\]
		where $\gamma$ and $\beta$ are learned parameters and addition and division are element-wise.
	\item $\MLP$ is a one-layer feed-forward network with ReLU activation. The $\MLP$ is row wise of $X$:
		\[
			\MLP(X[i]) = \relu(X[i, :] \cdot A_{enc} + b_{enc}) \cdot A_{dec}
		\]
		Depending on the model there may be additional bias terms added after the ReLU activation.
		For simplicitily, we will only consider one bias term prior to the ReLU activation though the results of this paper can be easily be extended to two bias terms.
	\item $\RoPE$ \cite{su2024roformer} is the rotary position encoding which applies a rotation to key and query vectors. For input $X$ at row $i$, the rotation is applied as follows: 
		\[
			\RoPE(X)[i, 2j] = \cos(i\theta_j)X[i, {2j}] - \sin(i\theta_j)X[i, 2j + 1]
		\]
		and
		\[
			\RoPE(X)[i, 2j + 1] = \sin(i\theta_j)X[i, {2j}] + \cos(i\theta_j)X[i, 2j + 1]
		\]
		where $\theta_j = 10000^{-2j/\dEmb}$ is the frequency for dimension $j$.
	\item $\attnH$ is an attention head for matrix $X \in \R^{\nctx \times \dEmb}$, an attention head does the following:
		\[
			\attnH(X) = \softmax\left( \frac{\RoPE(X Q) \cdot \RoPE(K^T X^T)}{\sqrt{\dEmb / H}} \right) \cdot X V
		\]
		where the softmax is applied column-wise.
        As in Ref.~\cite{black2022gptneox20bopensourceautoregressivelanguage}, we can re-write the effect of RoPE via a rotation matrix $\Theta_{i, j}$ such that
        \[
            \left(\RoPE(X Q) \cdot \RoPE(K^T X^T)\right)[i, j] = \vece_i \cdot X Q \cdot \Theta_{i, j} \cdot  K^T  X^T \cdot \vece_j^T
        \]
	Further, we will restrict our attention to \emph{causal} attention, which means that the attention matrix is upper triangular.
	Formally,
	\[
		\attnH(X) = \softmax\left( \frac{\RoPE(X Q) \cdot \RoPE(K^T X^T)}{\sqrt{\dEmb / H}} + M \right) \cdot X V
	\]
	where
	\[
		M_{ij} = \begin{cases}
			0 & i \leq j \\
			-\infty & i > j.
		\end{cases}
	\]
	\item Often, we have a multi-head attention mechanism which is the concatenation of $H$ attention mechanisms along the last dimension:
		\[
			\attn(X) = [\attnH_1(X); \ldots; \attnH_H(X)]
		\]
		where $\attnH_h$ is the $h$-th attention mechanism as outlined
	\item $\Embed \in \R^{\dVocab \times \dEmb}$ is the embedding function which maps a one hot vector in $\R^\dVocab$ to a vector in $\R^{\dEmb}$.
	\item $\Unembed \in \R^{\dEmb \times \dVocab}$ is the unembedding function which maps a vector in $\R^{\dEmb}$ to a one hot vector in $\R^\dVocab$.
\end{itemize}

%% file: sections/appendix_framework.tex
\section{Proofs for Worst-Case Deviation}
\label{sec:proofs_worst_case_deviation}

\begin{lemma}[Properties of the Worst-Case Deviation]
	\label{lemma:worst_case_deviation_properties}
	For any functions $f, g: \calX \to \R^n$, norm $p$ and lift to $\calY \times \calZ$, we have the following properties:
	\begin{itemize}[nosep]
		\item Triangle inequality for (lifted) $\WD$:	
			\[
				\WD(f + g; \calX)_p \leq \WD(f : \calX)_p + \WD(g : \calX)_p.
			\]
		\item Lipschitz composition:
			For function $g$
               \[
				\WD(g \circ f; \calX)_p \leq \Lip(g)_p \cdot \WD(f; \calX)_p.
			\]
		As a corollary, we have that for linear operators $A$,
		\[
			\WD(A f; \calX)_p \leq \norm{A}_p \cdot \WD(f; \calX)_p.
		\]
		as $\Lip(A)_p = \norm{A}_p$.

		\item $p$-norm bounds for $p, q \geq 1$ and $q > p$:
			\[
				\WD(f; \calX)_q \leq \WD(f; \calX)_p
			\]
	\end{itemize}
\end{lemma}

\begin{proof}[Proof of worst-case deviation properties, \cref{lemma:worst_case_deviation_properties}]
	\label{proof:worst_case_deviation_properties}
	We will prove each of the properties in turn.
	\begin{itemize}
		\item Triangle inequality:
			We can view the maximization over $\calX$ as occuring disjointly for $f$ and $g$:
			I.e. \begin{align*}
				\WD(f + g; \calX)_p 
			&\leq
            \sup_{X_1, X_2 \in \mathcal{X}} \norm{
			f(X_1) + g(X_1) - (f(X_2) + g(X_2))} \\
			&\leq \sup_{X_1, X_2, X_1', X_2' \in \mathcal{X}}  \big[\norm{f(X_1) - f(X_2)} 
            + \norm{g(X_1') - g(X_2')}\big] 
			\tag{by triangle inequality of norms}\\
			&\leq \WD(f ; \calX)_p + \WD(g ; \calX)_p
			\end{align*}
			as desired.
		\item Lipschitz composition:
			We simply have that
			\begin{align*}
				&\WD(A f; \calX)_p = \sup_{X_1, X_2 \in \calX} \norm{A f(X_1) - A f(X_2)}_p \\
						  &= \sup_{X_1, X_2 \in \calX} \norm{A (f(X_1) - f(X_2))}_p \\
						  &\leq \sup_{X_1', X_2'} \frac{\norm{A(X_1') - A(X_2')}_p}{\norm{X_1' - X_2'}_p} \cdot \sup_{X_1, X_2 \in \calX} \norm{f(X_1) - f(X_2)}_p \\
						  &=\Lip(A)_p \cdot \WD(f; \calX)_p.
			\end{align*}
		\item $p$-norm bounds for $p, q \geq 1$ and $q > p$:
			Because we restrict $f$ to be a function which outputs vectors and $\norm{\vec{x}}_q \leq \norm{\vec{x}}_p$ for $q > p$, we have that for all $X_1, X_2 \in \calX$, $\norm{f(X_1) - f(X_2)}_q \leq \norm{f(X_1) - f(X_2)}_p$.
			So, if there exists $X_1, X_2 \in \calX$ such that $\norm{f(X_1) - f(X_2)}_q = \alpha$, then there must exist $X_1, X_2 \in \calX$ such that $\norm{f(X_1) - f(X_2)}_p \geq \alpha$.
			Thus, the supremum over $\calX$ for $q$ is less than or equal to the supremum over $\calX$ for $p$.
		\end{itemize}
	\end{proof}

%% file: sections/appendix_techincal.tex
\section{Proofs for \cref{sec:meta_framework}}
\label{sec:proofs_framework}

\subsection{Proofs and Subalgorithms for Attention Bounds}
\begin{definition}[$\ell^{min}$, $\ell^{max}$] \label{def:lminlmax}
	We use $\ell^{\min}_i$ to denote 
	a worst-case lower bound on the smallest possible logit at the $i$-th position of the input to the softmax in model $\Model$. 
	Similarly, we use $\ell^{\max}_i$ to denote a worst-case upper bound on the largest possible logit at the $i$-th position of the input to the softmax in model $\Model$.
\end{definition}

We provide the algorithm to find the extremal values $\vec{\ell}^{\min}$ and $\vec{\ell}^{\max}$ in \cref{alg:lminlmax}.
At a high level, the algorithm computes upper and lower bounds for each position in the input sequence prior to the softmax.

\begin{algorithm}[h] 
	\caption{
        Algorithm to find $\vec{\ell}^{\min}$ and $ \vec{\ell}^{\max}$.
}
	\label{alg:lminlmax}
	\begin{algorithmic}
		\STATE {\bfseries Input:}  $\blowupSet\shiftSet, \Model$
		\STATE
		\FOR{$k \in [\nfix]$}
		\STATE Let \begin{align*}
			\vec{\ell}^{\min}_k = \vec{\ell}^{\max}_k = \frac{1}{\sqrt{\dEmb}} \vece_{q} \EmbedLN  \cdot Q \cdot \PosRot_{k, \nctx} 
		K^T \cdot  \EmbedLN^T \cdot \vece_{\desSet_k}^T
		\end{align*}
		\ENDFOR
		\STATE
		\FOR{$k \in (\nfix, \nctx)$}
		\STATE Let \[
		\vec{\ell}^{\min}_k = \frac{1}{\sqrt{\dEmb}} \min_{t \in [\dVocab]} \vece_q \cdot \EmbedLN \cdot Q \PosRot_{k, \nctx} \cdot K^T \cdot  (\vece_t \cdot \EmbedLN)^T
		\]
		and 
		\[
		\vec{\ell}^{\min}_k = \frac{1}{\sqrt{\dEmb}} \max_{t \in [\dVocab]} \vece_q \cdot \EmbedLN \cdot Q \PosRot_{k, \nctx} \cdot K^T \cdot  (\vece_t \cdot \EmbedLN)^T
		\]
		\ENDFOR
		\STATE
		\STATE Set the extremal values of the query token:
		\[
		\vec{\ell}^{\min}_\nctx = \vec{\ell}^{\max}_\nctx  = \frac{1}{\sqrt{\dEmb}} (\vece_{q} \cdot \EmbedLN) \cdot Q \cdot \PosRot_{k, \nctx} K^T \cdot  (\vece_{q} \cdot \EmbedLN )^T
		\]
		\STATE
		\STATE Return $\vec{\ell}^{\min}, \vec{\ell}^{\max}$
	\end{algorithmic}
\end{algorithm}


\begin{proof}[Proof sketch for \cref{lem:lminlmax} (correctness of \cref{alg:lminlmax})]
	In \cref{alg:lminlmax}, for each position $k$, we compute the minimum and maximum logit for the $k$-th position by maximizing over possible input tokens for the free tokens.
	Note that the fixed and query tokens have constant logit value, and so the minimum and maximum logit values equal each other.
%
\end{proof}

We now prove \cref{lem:min_max_softmax}, restated here:

\begin{lemma}[Minimum and Maximum after Softmax]
	\begin{equation}
		\label{eq:softmax_upper}
		\ssMinFB \leq
		\sum_{j \in (\nfix, \nctx)} \softmax(\ell_j) \leq
		\ssMaxFB
	\end{equation}
	as well as,
	\begin{align}
		\label{eq:softmax_lower}
		\ssMinRB \leq
		\sum_{j \in [\nfix] \cup \{\nctx\}} \softmax(\ell_j) \leq
		\ssMaxRB.
	\end{align}
\end{lemma}
\begin{proof}[Proof of \cref{lem:min_max_softmax}]
	We will prove the right hand side of \cref{eq:softmax_upper} and the left hand side of \cref{eq:softmax_lower} as the other follow by the same proof.
	Note the second inequality follows from the first inequality and the fact that the sum of the attention weights is $1$.

	Now, we will show the first inequality.
	Clearly, smaller values of $\vec{\ell}_i$ results in larger value of $\softmax(\vec{\ell}_j)$.
	Then, let $\vec{\eta}$ be a vector where $\sum_{j \in (\ndes, \nctx]} \softmax(\vec{\eta}_j) > \sum_{j \in (\ndes, \nctx]} \softmax(\vec{\ell}_j)$.
	Then, we have that
	\begin{align*}
		0 &\leq \left[\left(\sum_j e^{\vec{\eta}_j}\right) - \left(\sum_j e^{\vec{\ell}_j}\right)\right] \left(\sum_i e^{\vec{\ell}_i}\right)  \\
		\Rightarrow &\left(\sum_j e^{\vec{\eta}_j} \right)\left(\sum_i e^{\vec{\ell}_i} + \sum_j e^{\vec{\ell}_j} \right) \leq
        \left(\sum_j e^{\vec{\ell}_j}\right) \left(\sum_i e^{\vec{\eta}_i} + \sum_j e^{\vec{\eta}_j}\right) \\
		\Rightarrow & \frac{\sum_j e^{\vec{\eta}_j}}{\sum_i e^{\vec{\ell}_i} + \sum_j e^{\vec{\eta}_j}} \leq 
        \frac{\sum_j e^{\vec{\ell}_j}}{\sum_i e^{\vec{\ell}_i} + \sum_j e^{\vec{\ell}_j}}.
	\end{align*}
\end{proof}

\newcommand{\free}{\text{free}}
\newcommand{\fix}{\text{fix}}
Finally, we prove the bound on worst-case deviation of attention (\cref{lem:att_bound}):
\begin{lemma}
	Worst-case deviation of attention is bounded as follows,
	\begin{align*}
		\WD(\vecNctx \cdot f^{\attnH}_{\ndes \mid r, q} ; \InpSpace)_\infty &\leq
	     (\ssMaxRB - \ssMinRB)
          \cdot \norm{\EmbedLN[[\nfix] \cup \{\nctx\}, :] \cdot V \cdot \Unembed}_{\frobInf}
		+ 2 \cdot \ssMaxFB \cdot \norm{\EmbedLN \cdot V \cdot \Unembed}_{\frobInf}.
	\end{align*}
        where $V$ is the value matrix in the attention head (see \cref{sec:appendix_model}).
\end{lemma}
\begin{proof}[Proof of \cref{lem:att_bound}]
	\label{proof:att_bound}
	For notational simplicity, we will absorb the $\Unembed$ matrix into the value matrix $V$.
	So, here we will write $V = V' \cdot \Unembed$ where $V'$ is the value matrix in the attention head.

	Define $\vec{p}(X) \in \R^{\nctx}$ as the probability vector for the query token post-softmax.
	I.e.
	\[
		\vec{p}_j = \softmax(\vec{\ell}_j) \text{ for } j \in [\nctx]
	\]
	where \[
		\vec{\ell} = \frac{\vecNctx \cdot(Y \cdot \RoPE(Q) \RoPE(K^T) Y^T}{\sqrt{\dEmb}}
	\]
	for $Y = X \EmbedLN$.
	Moreover, let $\vec{p}(X)_{\free} = \vec{p}(X)[(\nfix, \nctx)]$ (i.e.\ the values corresponding to the free tokens) and $\vec{p}(X)_{\fix} = \vec{p}(X)[[\nfix] \cup \{\nctx\}]$ (the values corresponding to the fixed and query token).
	Now, we re-write the worst-case deviation of the attention head: for $X, X' \in \InpSpace$, 
	\begin{align*}
		\WD&(\vecNctx \cdot f^{\attnH}_{\ndes \mid r, q} ; \InpSpace \times \blowupSet \shiftSet)_\infty
			\leq  \max_{X, X'} 
			\norm{\vec{p}(X) \cdot (X \EmbedLN \cdot V) - \vec{p}(X') \cdot (X' \cdot V}_\infty \tag{by definition of an attention head} \\
		   &\leq \max_{X, X'} \bigg\|\vec{p}(X)_\fix \cdot (X[[\nfix] \cup \{\nctx\}:, ] \EmbedLN ) V + \vec{p}(X)_\free \cdot (X[(\nfix, \nctx):, ] \EmbedLN) V
		\\ &- \vec{p}(X')_\fix \cdot (X'[[\nfix] \cup \{\nctx\}:, ] \EmbedLN) V - \vec{p}(X')_\free \cdot (X'[(\nfix, \nctx):, ] \EmbedLN ) V \bigg\| \\
		   &\leq \max_{X, X'} \bigg\|\vec{p}(X)_\fix \cdot (X[[\nfix] \cup \{\nctx\}:, ] \EmbedLN) V - \vec{p}(X')_\fix \cdot (X'[[\nfix] \cup \{\nctx\}:, ] \EmbedLN \cdot ) V\bigg\| \\
		   &+ \bigg\|\vec{p}(X)_\free \cdot (X[(\nfix, \nctx):, ] \EmbedLN) V - \vec{p}(X')_\free \cdot (X'[(\nfix, \nctx):, ] \EmbedLN ) V \bigg\|.
		   \tag{by triangle inequality} \\
	\end{align*}
	Now, we will bound the two norms in the above equation separately.
	First, note that $\ssMinFB \leq \sum \vec{p}(X)_\free \leq \ssMaxFB$ and $\ssMinRB \leq \sum \vec{p}(X)_\fix \leq \ssMaxRB$ by definition of the extremal values (\cref{def:soft-extrem-values}).
	For the case of the free tokens,
	\begin{align*}
		&\max_{X, X'} \bigg\|\vec{p}(X)_\free \cdot (X[(\nfix, \nctx):, ] \EmbedLN) V - \vec{p}(X')_\free \cdot (X'[(\nfix, \nctx):, ] \EmbedLN) V \bigg\|_\infty
		\\
		&\leq 2 \cdot \max_{X} \bigg\|\vec{p}(X)_\free \cdot (X[(\nfix, \nctx):, ] \EmbedLN) V \bigg\|_\infty \tag{by triangle inequality} \\
		&\leq 2  \sum_{i \in (\nfix, \nctx)} \vec{p}(X)_\free[i]  \cdot \left\| (X[i, :] \EmbedLN) V \right\|_\infty \tag{by  triangle inequality} \\
		&\leq 2 \cdot \max_i \ssMaxFB \cdot \left\|(X[i, :] \EmbedLN) V \right\|_\infty \tag{by definition of $\vec{p}(X)_\free$} 
	\end{align*}
	where the last inequality follows from the fact that $\vec{p}(X)_\free$ is non-negative and the sum is at most $\ssMaxFB$.
	Finally, note that $\max_i \ssMaxFB \| X[i, :] \EmbedLN) V\| \leq \ssMaxFB\max_{t \in [\dVocab]} \|\Embed[t, :] \cdot V\|$ as the maximum free token contribution is at most the maximum contribution from any possible token.
	Finally, for the free tokens, we get a bound
	\[
		2 \cdot \ssMaxFB \cdot \max_{t \in [\dVocab]} \norm{(\EmbedLN[t]) V}_\infty 
		= 2 \cdot \ssMaxFB \cdot \norm{(\EmbedLN) V}_\frobInf.
	\]
	We now bound the fixed tokens contribution:
	\[
		\max_{X, X'} \bigg\|\vec{p}(X)_\fix \cdot (X[[\nfix] \cup \{\nctx\}:, ] \EmbedLN) V - \vec{p}(X')_\fix \cdot (X'[[\nfix] \cup \{\nctx\}:, ] \EmbedLN) V\bigg\|
	\]
	We start by noting that the fixed tokens must be an element of $\desSet \cup \{q\}$ and are, by definition, fixed for all $X$
	And so,
	\begin{align*}
		&\max_{X, X'} \bigg\|\vec{p}(X)_\fix \cdot (X[[\nfix] \cup \{\nctx\}:, ] \EmbedLN) V - \vec{p}(X')_\fix \cdot (X'[[\nfix] \cup \{\nctx\}:, ] \EmbedLN \cdot ) V\bigg\| \\
		&= \max_{X, X'} \bigg\|(\vec{p}(X)_\fix \cdot (\EmbedLN[s \cup \{q\}, :]) V - \vec{p}(X')_\fix \cdot (\EmbedLN[s \cup \{q\}, :]) V\bigg\| \\
		&\leq \max_{X, X'} \sum_i \bigg\|\vec{p}(X)_\fix[i] \cdot (\EmbedLN[t_i, :]) V - \vec{p}(X')_\fix[i] \cdot (\EmbedLN[t_i, :]) V\bigg\| \tag{by the triangle inequality}
	\end{align*}
	where $t_i$ is the $i$-th fixed token.
	And then, we get
	\begin{align*}
		&\max_{X, X'} \sum_i \bigg\|\vec{p}(X)_\fix[i] \cdot (\EmbedLN[t_i, :]) V - \vec{p}(X')_\fix[i] \cdot (\EmbedLN[t_i, :]) V\bigg\| \\	
		&= \max_{X,X'} \sum_i \bigg\|\left(\vec{p}(X)_\fix[i] - \vec{p}(X')_\fix[i]\right) \cdot (\EmbedLN[t_i, :]) V \bigg\|\\
		&= \max_{X, X'} \sum_i \left(\vec{p}(X)_{\fix}[i] - \vec{p}(X')_{\fix}[i]\right) \cdot \bigg\|(\EmbedLN[t_i, :]) V \bigg\|\\
		&\leq \max_{i} (\ssMaxFB - \ssMinFB) \cdot \bigg\|(\EmbedLN[t_i, :]) V \bigg\|
	\end{align*}
	where the last inequality follows from the fact that $\vec{p}(X)_\fix[i] - \vec{p}(X')_\fix[i]$ is at most $\ssMaxFB - \ssMinFB$.
	Then, by definition of the $\frobInf$ norm, we have that
	\[
		\max_{i} (\ssMaxFB - \ssMinFB) \cdot \bigg\|(\Embed[t_i, :]) V \bigg\| = (\ssMaxFB - \ssMinFB) \cdot \max_{B, S} \norm{(\Embed[s \cup \{q\}, :]) V}_\frobInf.
	\]

	Putting the above together, we get the desired bound:
	\begin{align*}
		\WD(\vecNctx \cdot f^{\attnH}_{\ndes \mid r, q} ; \InpSpace \times \blowupSet \shiftSet)_\infty
	     &\leq
	     (\ssMaxRB - \ssMinRB) \cdot \max_{B, S}\norm{\Embed[\desSet \cup \{q\}, :] \cdot V}_{\frobInf} \\
	     &+ 2 \ssMaxFB \cdot \max_{B, S} \norm{(\EmbedLN) \cdot V}_{\frobInf}.
	\end{align*}
\end{proof}


\subsection{Proof of \cref{thm:InpRes}}
\label{subsec:InpResProof}
First we will restate \cref{thm:InpRes} for convenience:
\begin{theorem}
	If
	\begin{align*}
	    &\WD(\desF{\Model}; \InpSpace)_\infty
	  \\&< \peakToPeak(\desF{\Model}, X) / 2,
	\end{align*}
	then the output of model $\desF{\Model}$ is fixed for all inputs in $\InpSpace$.
	Moreover, we can use \cref{alg:overwhelmCheckDet} to produce an upper bound $W$ for $\WD(\desF{\Model}; \InpSpace)_\infty$.
\end{theorem}
\begin{proof}
	We first make use of \cref{thm:metathm} to prove the first statement of the above theorem.
	Now, we just need to prove that the bound, $W$ in \cref{alg:overwhelmCheckDet}, is a valid bound on the lift $\WD(\desF{\Model}; \InpSpace)_\infty$.

	First note that, by \cref{lemma:worst_case_deviation_properties}, we have that
	\begin{align*}
		\WD(\desF{\Model}; \InpSpace)_\infty &\leq \Lip(\Unembed)_\infty \cdot (\WD(f^{\attnH}; \InpSpace)_\infty \\ &+ \WD(f^{\MLP}; \InpSpace)_\infty + \WD(f^{\Iden}; \InpSpace)_\infty) 
	\end{align*}
	where
	\[
		f^{\component}(X) = \component \circ \fenc(X) 
	\]
	and
	\[
		\fenc(X) = (X \cdot \EmbedLN).
	\]
	Then, we can use \cref{lem:att_bound} to get that $W^\attn$ is a valid bound on the worst-case deviation of $f^{\attnH}$.
	Then, we use \cref{lem:mlpbound} to get that $\Lip(\MLP)_\infty \cdot \WD(\fenc; \InpSpace \times \blowupSet\shiftSet)$ is a valid bound on the worst-case deviation of $f^{\MLP}$.
\end{proof}

%% file: sections/appendix_perm.tex
\section{Appendix for Permutation Invariance}
\label{sec:appendix_perm}

\begin{algorithm}[tb] 
	\caption{Naive Algorithm to find $\ssMinFA$.\\
		To find $\ssMaxFA$, switch the $\min$ to a $\max$ in the objective.
	}
	\label{alg:naiveAlpha}
	\begin{algorithmic}
	\STATE \textbf{Input:} {$\fenc$ and model $\Model$}
	Let $\fenc(\vec{e}) = \vec{e} \cdot \Embed \cdot \diag(B) + S$ \\
	\FOR{$k \in (\ndes, \nctx]$}
	\STATE {
		Set \begin{align*}
			\vec{\ell}^{\min}_k = \frac{1}{\sqrt{\dEmb}} \min_{t \in \freeToks} 
			\fenc(\vecQuery) \cdot Q \PosRot_{k, \nctx} K^T \cdot \fenc(\vec{e}_t)^T
		\end{align*}
	}
	\ENDFOR
	\STATE Return $\ssMinFA = \sum_k \exp\left(\vec{\ell}^{\min}_k\right)$
	\end{algorithmic}
\end{algorithm}

\subsection{Proof of Attention Worst-Case Deviation}
In this subsection, we provide a proof of \cref{lem:corrAlpha} and \cref{lem:attn_perm}.
First, we prove the correcntess of the algorithms to find $\ssMinFA$ and $\ssMaxFA$.
\begin{lemma}[Restatement of \cref{lem:corrAlpha}]
	 Both \cref{alg:LPAlpha} and \cref{alg:naiveAlpha} provide valid bounds on $\ssMinFA$ and $\ssMaxFA$.
\end{lemma}
\begin{proof}
	We provide a proof for $\ssMinFA$; the proof for $\ssMaxFA$ is analogous.

	First, note that \cref{alg:naiveAlpha} is the same as \cref{alg:lminlmax} except that we restrict our free-token attention to those specified by the permutation class $[X]$.
	Thus, the correctness of \cref{alg:lminlmax} implies that \cref{alg:naiveAlpha} is correct.

	For the LP-based algorithm, \cref{alg:LPAlpha}, we first consider the LP as an integer linear program.
	Consider the permutation matrix $X$ with entries $x_{j, t}$, where $x_{j, t} = 1$ if token $t$ is selected for position $j$ and $x_{j, t} = 0$ otherwise.
	Then, the LP can be seen as finding the permutation matrix which minizes the objective function, $\ssMinFA$.
	So, the relaxed LP is a lower bound on the integer LP which is a lower bound on the true value of $\ssMinFA$.
\end{proof}

We now restate the lemma for the worst-case deviation of attention.
\begin{lemma}
	Worst-case deviation of attention is bounded as follows when considering the permutation class $[X]$ with free tokens $\freeToks$:
	\begin{align*}
		&\WD(\vecNctx \cdot f^{\attnH}_{\ndes \mid r, q} ; \InpSpace)_\infty
	     \\ &\leq
		(\ssMaxRB - \ssMinRB) \cdot \norm{\EmbedLN[[\nfix] \cup \{\nctx\}, :] \cdot V \cdot \Unembed}_{\frobInf} \\
		&+ 2 \cdot \ssMaxFB \cdot \norm{(\EmbedLN[(\nfix,\nctx),: ] \cdot V \cdot \Unembed}_{\frobInf}.
	\end{align*}
\end{lemma}
\begin{proof}
	Again, for notational simplicity, we will absorb the $\Unembed$ matrix into the value matrix $V$.
	We will write $V = V' \cdot \Unembed$ where $V'$ is the value matrix in the attention head.
	The proof is the same as the proof of \cref{lem:att_bound} (bound on worst-case deviation of attention in the generic case).
\end{proof}

\subsection{Proof of \cref{thm:InpResPermInv}}
\label{subsec:proofInpResPermInv}
We first restate the theorem for convenience.
\begin{theorem}
	If
	\[
		\WD(\desF{\Model}; [X])_\infty < \peakToPeak(\desF{\Model}, X) / 2
	\]
	then the output of model $\desF{\Model}$ is fixed for all inputs in $[X]$.
	Moreover, we \cref{alg:overwhelmCheck2} in \cref{sec:appendix_perm} produces an upper bound $W$ for $\WD(\desF{\Model}; \InpSpace)_\infty$.
\end{theorem}
\begin{proof}
	The proof follows analogously to the proof of \cref{thm:InpRes} in \cref{subsec:InpResProof} except that we consider the permutation class $[X]$ and thus can use either the algorithms \cref{alg:naiveAlpha} or \cref{alg:LPAlpha} to find bounds on $\ssMinFA$ and $\ssMaxFA$.
\end{proof}

%% file: sections/appendix_det_eps.tex
\section{Overwhelming for $\nctx \to \infty$}
\label{sec:convergence}

In this section, we will consider a specific set $\desSet = \bigtimes_{\nfix} \vecRep$ and $\query = \vecRep$ for some fixed $r \in [\dVocab]$.
In words, we will consider the fixed input to be one repeated token.
We term this a ``repetition restriction.''

Moreover, it will be useful to define the set of all possible inputs under a repetition restriction.
\begin{definition}[Repetition Space]
	\label{def:RepSpace}
	Let $\RepSpace \subset \OneHotSpace^\nctx$ be the set of all possible inputs under a repetition restriction.
	That is, 
	\[
		\RepSpace = \left\{ X \in \R^{\nctx \times \dVocab} \mid X = 
		\begin{bmatrix}
			\vecRep^T \\
			\vecRep^T \\
			\vdots \\
			\vecRep^T \\
			Y \\
			\vecQuery
		\end{bmatrix}, Y \in \OneHotSpace^{\nfree}
		\right\}.
	\]
\end{definition}

For simplicity, we will also not consider positional encodings in this section. 
I.e.\ we remove the use of RoPE in the attention mechanism.
Though as Ref.~\cite{barbero2024transformers} pointed out, rotary positional encodings \cite{su2024roformer} converge to providing zero information as $\nctx \to \infty$.

\begin{theorem}[Asymptotic Convergence to a Fixed Model]
	\label{thm:convergence}
	If $\frac{\nfree}{\nctx} \in o(1)$ as a function of $\nctx$, then the repetition restriction converges to a fixed model if $\peakToPeak(\Model, X)$ is positive for $X = \vece_r$.
\end{theorem}

To prove this theorem, we will need to (a) find a way to compute a $\peakToPeak$ function of the model and input \emph{independent of} $\nctx$ and (b) use the framework of \cref{sec:meta_framework} to individually bound the worst-case deviation for each component of the model.

\subsection*{Computing Peak-to-Peak Difference}
Given that $\peakToPeak$ is computed by evaluating the model on a single input, finding a $\peakToPeak$ value is normally a simple task.
But, in the asymptotic case, we need to ``shortcut'' the computation of $\peakToPeak$ for an $X \in \RepSpace$.
We can simply do this by setting all the free tokens to the repetition token: $X = {\vecRep}^{\nfree}$.

\begin{lemma}[Shortcut for $\peakToPeak$]
	\label{lem:gap_shortcut}
	We can compute $\peakToPeak(\Model, X)$ in time $O(1)$ for $X=\vece_\rep^{\nctx}$ even as $\nctx \rightarrow \infty$.
\end{lemma}
\begin{proof}
	So, we can easily compute blowup, $B$, and shift, $S$.
        Note, for attention head $\vece_\nctx \cdot \desF{f^\attnH}$, the output equals 
        \[
            \vec{p}(X) \cdot X \EmbedLN \cdot V
        \]
        for some probability vector $\vec{p}(X)$.
        Because $X = \vecRep^\nctx$, $X \EmbedLN \cdot V= \vec{a}^\nctx$ for vector $\vec{a} = (\vecRep \cdot \EmbedLN) V$.
        Thus,
        $$
            \vec{p}(X) \cdot (X \EmbedLN) \cdot V = \vec{a}.
        $$
	We therefore get that the output of $\vece_\nctx \cdot \desF{f^\attnH}$ is fixed to $\vecRep \EmbedLN \cdot V$ for all context windows $\nctx$.
	Finally, we can simply compute the value of $\vece_\nctx \cdot \desF{f^\MLP}$ because we know $B$ and $S$ and so just need to compute $\MLP(\vec{e}_r \EmbedLN)$.
\end{proof}

\subsubsection*{Bound on Attention Head Difference}
To bound attention, we will take advantage of the repeated structure as well as the ideas in \cref{lem:min_max_softmax} (bounds on the attention weights).

\begin{corollary}[Sum of Attention Weights]
	\label{cor:sum_attention}
	Let $\logRMin \leq \min_{i \in [\ndes] \cup \{n_ctx\}} \ell_i$ and $\logFMin \leq \min_{j \in (\nfix, \nctx)} \ell_j$.
	Also, let $\logRMax \geq \max_{i \in [\ndes] \cup \{n_ctx\}} \ell_i$ and $\logFMax \geq \max_{i \in (\nfix, \nctx)} \ell_i$. Then,
	\begin{align*}
	\frac{\ndes}{\ndes + 1 + (\nfree) \cdot e^{\logFMax - \logRMin}} 
\leq 
		\sum_{i \in [\ndes] \cup \{\nctx\}} \softmax(\vec{\ell}_i)
	\leq
		\frac{\ndes}{\ndes + 1 + (\nfree) \cdot e^{\logFMin - \logRMax}}
	\end{align*}
	and 
	\begin{align*}
		\frac{\nfree}{\ndes + 1 + (\nfree) \cdot e^{\logFMin - \logRMax}} 
        \leq \sum_{j \in (\nfix, \nctx)} \softmax(\vec{\ell}_j)
        \leq \frac{\nfree}{\ndes + 1 + (\nfree) \cdot e^{\logFMax - \logRMin}}.
	\end{align*}
\end{corollary}

\begin{lemma}[Attention Bound, \cref{lem:att_bound}]
	\label{lem:convgattn}
	For large enough $\nctx$, we get
	\[
		\WD(\desF{f^\attnH}; \RepSpace)_\infty \leq o(1).
	\]
\end{lemma}
\begin{proof}
	\label{proof:convgattn}
	Recall that \cref{lem:att_bound} gives us
	\begin{align*}
		&\WD(\vecNctx \cdot f^{\attnH}_{\ndes \mid r, q} ; \InpSpace)_\infty  
						 \leq					   (\ssMaxRB - \ssMinRB) \cdot \norm{\vecDes \EmbedLN} + 
											   2 \cdot \ssMaxFB \cdot \\norm{(\vecDes \EmbedLN ) \cdot V}_\infty
	\end{align*}
	Because $\logFMax - \logRMin \leq \theta$ and $\logFMin - \logRMax \leq \theta$,
	we can use \cref{cor:sum_attention} to upper-bound $\ssMaxRB - \ssMinRB$ and $\ssMaxFB$ by
	\begin{align*}
		1 - \frac{\ndes}{\ndes + (\nfree + 1) \cdot \theta}
		= \frac{(\nfree + 1) \cdot \theta}{\ndes + (\nfree + 1) \cdot \theta} 
		\leq \frac{\nfree \cdot \theta}{\ndes} 
		\leq \theta \cdot o(1).
	\end{align*}
	So then
	\[
		\WD(\desF{f^{\attnH}} ; \InpSpace)_\infty \leq 
		o(1) 
	\]
	as $\theta \in o(1)$.
\end{proof}
%
%
%
We are now ready to prove \cref{thm:convergence}.
\begin{proof}[Proof of \cref{thm:convergence}]
	Note that we reduced $\WD(\vec{e}_\nctx \cdot \desF{\Model}; \RepSpace)_\infty$ to be upper-bounded by 
	$
	O(1) \cdot \WD(\vec{e}_\nctx \cdot \desF{\fenc}; \RepSpace) + o(1)
	$
	through \cref{lem:mlpbound} and \cref{lem:convgattn}.
Then, by \cref{thm:InpRes}, we have that \cref{thm:convergence} holds as $\WD$ goes to $0$ as $\nctx \rightarrow \infty$.
    So, as long as $\peakToPeak(\desF{\Model}, X)$ is positive for \emph{some} $X \in \RepSpace$, then we converge to ``overwhelming'' by \cref{thm:metathm}.
    Note that $\vece_\rep^\nctx \in \RepSpace$ and by \cref{lem:gap_shortcut}, we can compute a sample of peak-to-peak deviation for all $\nctx$.
\end{proof}

%% file: main.bbl
\begin{thebibliography}{17}
\providecommand{\natexlab}[1]{#1}
\providecommand{\url}[1]{\texttt{#1}}
\expandafter\ifx\csname urlstyle\endcsname\relax
  \providecommand{\doi}[1]{doi: #1}\else
  \providecommand{\doi}{doi: \begingroup \urlstyle{rm}\Url}\fi

\bibitem[Alon \& Yahav(2021)Alon and
  Yahav]{alon2021bottleneckgraphneuralnetworks}
Alon, U. and Yahav, E.
\newblock On the bottleneck of graph neural networks and its practical
  implications, 2021.
\newblock URL \url{https://arxiv.org/abs/2006.05205}.

\bibitem[Barbero et~al.(2024)Barbero, Banino, Kapturowski, Kumaran, Ara{\'u}jo,
  Vitvitskyi, Pascanu, and Veli{\v{c}}kovi{\'c}]{barbero2024transformers}
Barbero, F., Banino, A., Kapturowski, S., Kumaran, D., Ara{\'u}jo, J.~G.,
  Vitvitskyi, A., Pascanu, R., and Veli{\v{c}}kovi{\'c}, P.
\newblock Transformers need glasses! information over-squashing in language
  tasks.
\newblock \emph{arXiv preprint arXiv:2406.04267}, 2024.

\bibitem[Black et~al.(2022)Black, Biderman, Hallahan, Anthony, Gao, Golding,
  He, Leahy, McDonell, Phang, Pieler, Prashanth, Purohit, Reynolds, Tow, Wang,
  and Weinbach]{black2022gptneox20bopensourceautoregressivelanguage}
Black, S., Biderman, S., Hallahan, E., Anthony, Q., Gao, L., Golding, L., He,
  H., Leahy, C., McDonell, K., Phang, J., Pieler, M., Prashanth, U.~S.,
  Purohit, S., Reynolds, L., Tow, J., Wang, B., and Weinbach, S.
\newblock Gpt-neox-20b: An open-source autoregressive language model, 2022.
\newblock URL \url{https://arxiv.org/abs/2204.06745}.

\bibitem[Devlin et~al.(2019)Devlin, Chang, Lee, and
  Toutanova]{devlin-etal-2019-bert}
Devlin, J., Chang, M.-W., Lee, K., and Toutanova, K.
\newblock Bert: Pre-training of deep bidirectional transformers for language
  understanding, 2019.
\newblock URL \url{https://arxiv.org/abs/1810.04805}.

\bibitem[Gross et~al.(2024)Gross, Agrawal, Kwa, Ong, Yip, Gibson, Noubir, and
  Chan]{gross2024compact}
Gross, J., Agrawal, R., Kwa, T., Ong, E., Yip, C.~H., Gibson, A., Noubir, S.,
  and Chan, L.
\newblock Compact proofs of model performance via mechanistic interpretability.
\newblock In \emph{ICML 2024 Workshop on Mechanistic Interpretability}, 2024.

\bibitem[Hahn(2020)]{hahn2020theoretical}
Hahn, M.
\newblock Theoretical limitations of self-attention in neural sequence models.
\newblock \emph{Transactions of the Association for Computational Linguistics},
  8:\penalty0 156--171, 2020.

\bibitem[Hahn \& Rofin(2024)Hahn and Rofin]{hahn2024sensitive}
Hahn, M. and Rofin, M.
\newblock Why are sensitive functions hard for transformers?
\newblock \emph{arXiv preprint arXiv:2402.09963}, 2024.

\bibitem[Kim et~al.(2021)Kim, Papamakarios, and
  Mnih]{kim2021lipschitzconstantselfattention}
Kim, H., Papamakarios, G., and Mnih, A.
\newblock The lipschitz constant of self-attention, 2021.
\newblock URL \url{https://arxiv.org/abs/2006.04710}.

\bibitem[Li et~al.(2022)Li, Choi, Chung, Kushman, Schrittwieser, Leblond,
  Eccles, Keeling, Gimeno, Dal~Lago, Hubert, Choy, de~Masson~d’Autume,
  Babuschkin, Chen, Huang, Welbl, Gowal, Cherepanov, Molloy, Mankowitz,
  Sutherland~Robson, Kohli, de~Freitas, Kavukcuoglu, and Vinyals]{Li_2022}
Li, Y., Choi, D., Chung, J., Kushman, N., Schrittwieser, J., Leblond, R.,
  Eccles, T., Keeling, J., Gimeno, F., Dal~Lago, A., Hubert, T., Choy, P.,
  de~Masson~d’Autume, C., Babuschkin, I., Chen, X., Huang, P.-S., Welbl, J.,
  Gowal, S., Cherepanov, A., Molloy, J., Mankowitz, D.~J., Sutherland~Robson,
  E., Kohli, P., de~Freitas, N., Kavukcuoglu, K., and Vinyals, O.
\newblock Competition-level code generation with alphacode.
\newblock \emph{Science}, 378\penalty0 (6624):\penalty0 1092–1097, December
  2022.
\newblock ISSN 1095-9203.
\newblock \doi{10.1126/science.abq1158}.
\newblock URL \url{http://dx.doi.org/10.1126/science.abq1158}.

\bibitem[Merrill \& Sabharwal(2023)Merrill and
  Sabharwal]{merrill2023parallelism}
Merrill, W. and Sabharwal, A.
\newblock The parallelism tradeoff: Limitations of log-precision transformers.
\newblock \emph{Transactions of the Association for Computational Linguistics},
  11:\penalty0 531--545, 2023.

\bibitem[O'Donnell(2014)]{o2014analysis}
O'Donnell, R.
\newblock \emph{Analysis of boolean functions}.
\newblock Cambridge University Press, 2014.

\bibitem[OpenAI(2024)]{openai2024gpt4technicalreport}
OpenAI.
\newblock Gpt-4 technical report, 2024.
\newblock URL \url{https://arxiv.org/abs/2303.08774}.

\bibitem[Peng et~al.(2024)Peng, Narayanan, and
  Papadimitriou]{peng2024limitations}
Peng, B., Narayanan, S., and Papadimitriou, C.
\newblock On limitations of the transformer architecture.
\newblock \emph{arXiv preprint arXiv:2402.08164}, 2024.

\bibitem[Sanford et~al.(2024)Sanford, Hsu, and
  Telgarsky]{sanford2024representational}
Sanford, C., Hsu, D.~J., and Telgarsky, M.
\newblock Representational strengths and limitations of transformers.
\newblock \emph{Advances in Neural Information Processing Systems}, 36, 2024.

\bibitem[Su et~al.(2024)Su, Ahmed, Lu, Pan, Bo, and Liu]{su2024roformer}
Su, J., Ahmed, M., Lu, Y., Pan, S., Bo, W., and Liu, Y.
\newblock Roformer: Enhanced transformer with rotary position embedding.
\newblock \emph{Neurocomputing}, 568:\penalty0 127063, 2024.

\bibitem[Team(2024)]{geminiteam2024geminifamilyhighlycapable}
Team, G.
\newblock Gemini: A family of highly capable multimodal models, 2024.
\newblock URL \url{https://arxiv.org/abs/2312.11805}.

\bibitem[Vaswani et~al.(2023)Vaswani, Shazeer, Parmar, Uszkoreit, Jones, Gomez,
  Kaiser, and Polosukhin]{vaswani2023attentionneed}
Vaswani, A., Shazeer, N., Parmar, N., Uszkoreit, J., Jones, L., Gomez, A.~N.,
  Kaiser, L., and Polosukhin, I.
\newblock Attention is all you need, 2023.
\newblock URL \url{https://arxiv.org/abs/1706.03762}.

\end{thebibliography}
